\documentclass[12pt]{article}

\usepackage{amssymb, amsmath, amsthm, amsfonts}
\usepackage{multirow}
\usepackage{booktabs}
\usepackage{subfigure}
\usepackage[authoryear]{natbib}
\usepackage[colorlinks,citecolor=blue,urlcolor=blue]{hyperref}
\usepackage{graphicx}
\usepackage{newtxtext}
\usepackage{epsfig,latexsym,verbatim}
\usepackage{bm}
\usepackage{mathrsfs}
\usepackage{multicol}
\usepackage{enumitem}
\usepackage{makecell}
\usepackage[figuresright]{rotating}
\usepackage{xr}
\graphicspath{{figures/}}
\usepackage{enumitem}
\usepackage{stackengine}
\usepackage{booktabs}
\usepackage{subfigure}
\usepackage{graphics}
\usepackage{lmodern}
\usepackage{bm}
\usepackage{tikz-cd}
\usepackage{comment}
\usepackage{algorithm,algpseudocode}
\usepackage{algorithmicx}
\usepackage{appendix}
\usepackage{float}
\usepackage{tcolorbox}
\newfloat{algorithm}{t}{lop}
\newcommand{\blind}{0}

\theoremstyle{plain}

\newtheorem{theorem}{Theorem}

\newtheorem{assumption}{Assumption}
\newtheorem{proposition}{Proposition}
\newtheorem{lemma}{Lemma}[section]
\newtheorem{remark}{Remark}

\theoremstyle{remark}
\newtheorem{definition}[theorem]{Definition}

\newcommand{\m}{{\mathcal M}}

\newcommand{\CO}{{\mathcal O}}

\newcommand{\R}{{\mathbb R}}
\newcommand{\Ric}{\operatorname{Ric}}
\newcommand{\diver}{\operatorname{div}}
\newcommand{\inj}{\operatorname{inj}}
\newcommand{\Exp}{\operatorname{Exp}}
\newcommand{\Log}{\operatorname{Log}}
\newcommand{\Sec}{\operatorname{sec}}

\newcommand{\Cut}{\operatorname{Cut}}

\begin{document}

\def\spacingset#1{\renewcommand{\baselinestretch}%
{#1}\small\normalsize} \spacingset{1}

\date{}
\if0\blind
{
  \title{\bf Geometric R\'enyi Differential Privacy: Ricci Curvature Characterized by Heat Diffusion Mechanisms}
  \author{Xiaotian Chang\footnotemark[1] \footnotemark[3],$\,$ Yangdi Jiang\footnotemark[1] \footnotemark[3], $\,$ Cyrus Mostajeran \footnotemark[1] and$\,$ Qirui Hu\footnotemark[2] \footnotemark[4] \hspace{.2cm}}
  \maketitle
  \renewcommand{\thefootnote}{\fnsymbol{footnote}}

  \footnotetext[1]{Nanyang Technology University}
  \footnotetext[2]{Shanghai University of Finance and Economics}
  \footnotetext[3]{Co-first authors. These authors contributed equally to this work.}
  \footnotetext[4]{Corresponding author: huqirui@mail.shufe.edu.cn}
}

\if1\blind
{
  \bigskip
  \bigskip
  \bigskip
  \begin{center}
    { \bf Geometric R\'enyi Differential Privacy: Ricci Curvature Characterized by Heat Diffusion Mechanisms}
  \end{center}
  \medskip
} \fi

\bigskip

\abstract{In this paper, we develop a novel privacy mechanism for Riemannian manifold-valued data. Our key contribution lies in uncovering unexpected connections among geometric analysis, heat diffusion models, and differential privacy (DP). We characterize the R\'{e}nyi divergence via dimension-free Harnack inequalities on Riemannian manifolds and establish R\'{e}nyi differential privacy guarantees governed by Ricci curvature. For manifolds with nonnegative Ricci curvature, we propose a mechanism based on heat diffusion. In contrast, for general manifolds we introduce a Langevin-process-based approach that yields intrinsic mechanisms supporting normalization-free sampling and continuous privacy–utility trade-offs. We derive detailed utility analyses for both mechanisms. As a statistical application, we develop privacy-preserving estimation of the generalized Fr\'{e}chet mean, including nontrivial sensitivity analysis and phase transition characterizations. Numerical experiments further demonstrate the advantages of the proposed DP mechanisms over existing approaches.}

\vspace{0.5cm}

\newpage
\spacingset{1.5}

\section{Introduction}\label{sec:intro}

Modern datasets increasingly contain observations that are not naturally represented as vectors in a Euclidean space.
Examples include directions represented as points on spheres, shapes and landmarks, rotations and other Lie-group-valued objects, covariance and diffusion tensors modeled as symmetric positive definite matrices, and more broadly ``random objects'' that live in general metric spaces.
This has motivated a rapidly growing statistical literature on inference and learning beyond linear spaces, where fundamental notions such as means, regression functions, and variability must be defined through geometry and distances; see, among many others,
Fr\'echet regression and its extensions \cite{petersen2019frechet,lin2021tvfrechet,chen2022uniform,iao2025deep},
recent developments in metric-space inference \cite{dubey2024metric},
and geometric statistics in applications such as medical imaging \cite{pennec2019riemannian}.
In these settings, curvature is not just a technical detail: it can affect stability and asymptotics of estimators, including central limit behavior of Fr\'echet means \cite{eltzner2019smeary}.

At the same time, non-Euclidean data are often derived from sensitive individual-level information (e.g., medical images, trajectories, and user behavior).
When releasing statistical summaries computed from such data, it is crucial to protect privacy.
Differential privacy (DP) provides a rigorous, worst-case notion of privacy that has become a standard in statistical disclosure limitation and privacy-preserving data analysis \cite{dwork2006differential,dwork2014algorithmic}.
However, most DP methodology and mechanisms are designed for Euclidean outputs.
For non-Euclidean domains, naive strategies that embed the data in an ambient Euclidean space, add noise, and project back can distort geometry, introduce coordinate artifacts, and even push outputs outside the admissible domain.
This is especially problematic for constrained spaces (e.g., SPD matrices or shape spaces), where respecting intrinsic geometry is often essential for downstream analysis.

\subsection{Motivation}\label{subsec:motivation}

A natural way to formulate privacy mechanisms on a manifold (or, more generally, on a metric space) is through probability distributions that directly live on the output space.
Recent work has begun to develop intrinsic privacy mechanisms for manifold-valued summaries, including Riemannian Laplace-type mechanisms \cite{reimherr2021dpmanifold},
shape/structure-preserving constructions \cite{soto2022shape},
and mechanisms specialized to non-Euclidean domains such as directional data \cite{weggenmann2021directional} and SPD manifolds \cite{utpala2023dpfrechetspd}.
Beyond pure $\varepsilon$-DP, manifold-aware privacy has also been extended to alternative privacy notions:
for instance, Gaussian differential privacy on Riemannian manifolds \cite{jiang2023gaussian},
data-density-aware mechanisms via conformal transformations \cite{he2025conformal},
and very recently exponential-wrapped Laplace/Gaussian mechanisms that achieve $\varepsilon$-DP, $(\varepsilon,\delta)$-DP, GDP and RDP on Hadamard manifolds without MCMC sampling \cite{jiang2026expwrap}.
Relatedly, objective perturbation approaches such as the K-norm gradient (KNG) mechanism \cite{reimherr2019kng} and its manifold adaptations have enabled privacy-preserving inference for structured summaries and learning problems, including shape/structure-preserving releases \cite{soto2022shape} and differentially private geodesic regression \cite{kulkarni2025dpgeodesic}.
There is also a broader viewpoint of metric differential privacy, where privacy guarantees are expressed in terms of distances on the output space \cite{chatzikokolakis2013metrics,imola2022metric}.
In parallel, optimization-based approaches inject noise intrinsically on tangent spaces to obtain privacy for manifold-constrained learning problems \cite{han2024dpriemannian}.

In this paper, we advocate a diffusion-based route.
Instead of postulating a closed-form output density, we define the privatization mechanism through a Markov semigroup:
given an input (non-private) summary $x\in \m$, we release
\[
Y \sim P_t(x,\cdot),
\]
where $(P_t)_{t\ge 0}$ is the transition semigroup of an intrinsic diffusion on the manifold $\m$.
Two canonical choices are the heat semigroup (Brownian motion) and a drifted diffusion (Langevin dynamics).
This perspective yields three practical advantages:
\begin{enumerate}[label=(\roman*)]
\item \emph{Intrinsic nature}: the mechanism is coordinate-free and equivariant under isometries;
\item \emph{Normalization-free sampling}: one can generate privacy noise by simulating an SDE \cite{JMLR:v26:24-0829} on $\m$, without evaluating a partition function;
\item \emph{A continuous privacy-utility dial}: the diffusion time $t$ interpolates between small-noise, high-utility outputs ($t\downarrow 0$) and more mixed, higher-privacy outputs (larger $t$), with the semigroup property enabling modular calibration via time.
\end{enumerate}

\par From the privacy-analysis standpoint, diffusion kernels are particularly compatible with R\'enyi differential privacy (RDP) \cite{mironov2017renyi}, which controls privacy loss via R\'enyi divergences \cite{vanerven2014renyi}.
RDP enjoys convenient composition properties and, crucially for our goals, interfaces naturally with semigroup techniques from geometric analysis.
In particular, dimension-free Harnack inequalities provide pointwise-to-integrated controls for diffusion semigroups, and are known to be equivalent to curvature lower bounds in the sense of Bakry-\'Emery \cite{wang2004equivalence}.
This suggests that curvature should play a structural role in privacy guarantees for diffusion mechanisms on manifolds. 

\subsection{Our contribution}\label{subsec:contrib}

We develop a theory of diffusion-based RDP on Riemannian manifolds and uncover a direct link between privacy budgets and curvature. Starting from a manifold-valued statistic $f(D)$, our mechanism releases a privatized output through the heat semigroup, equivalently through Brownian motion, via $B_t(f(D))\sim P_t(f(D),\cdot)$. The geometric input enters through the lower Ricci-curvature bound $\Ric\ge -K$, which governs the asymptotic behavior of the heat diffusion. The dimension-free Harnack inequality then converts this geometric control, together with the sensitivity $\Delta$ and the diffusion time $t$, into an explicit R\'enyi privacy budget. More concretely, the privacy-geometry dictionary takes the following form. Under positive curvature in the paper's sign convention ($K<0$), Brownian motion gradually forgets its starting point and the privacy budget decays to $0$ as $t\to\infty$. When $K=0$, one recovers the Euclidean Gaussian rate $\varepsilon=\alpha\Delta^2/(4t)$. Under negative curvature ($K>0$), the privacy budget cannot be driven to $0$ by increasing $t$ alone and instead approaches a strictly positive privacy floor, which explains why Brownian motion suffices in curvature-favorable settings. The left and right panels of Figure~\ref{fig:intro-logic-K-notation} summarize the mechanism-level roadmap of the paper and visualize the three qualitative privacy regimes induced by the curvature bound.

\begin{figure}[h!]
\centering
\includegraphics[width=\textwidth]{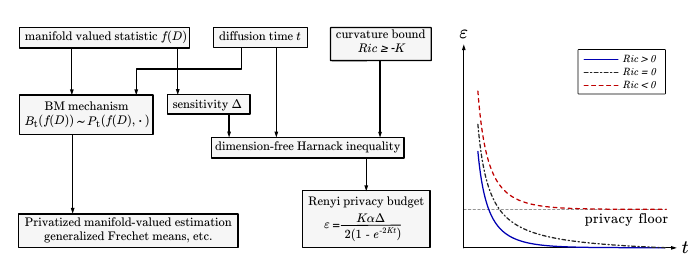}
\caption{Left: mechanism-level roadmap of the paper. Right: schematic behavior of the privacy budget $\varepsilon$ as a function of the diffusion time $t$ under the sign convention of $\Ric$.}
\label{fig:intro-logic-K-notation}
\end{figure}

On certain non-compact manifolds ($K>0$), Brownian motion may retain ``memory at infinity,'' preventing the privacy budget from improving beyond a positive floor even as $t\to\infty$.
To address this, we introduce a \emph{Langevin-process-based mechanism} that adds a confining drift.
On Hadamard manifolds, natural strongly geodesically convex potentials such as $V(x)=\lambda d(o,x)^2/2$ yield intrinsic mechanisms with normalization-free sampling and a continuous privacy-utility trade-off governed jointly by $(t,\lambda)$.

Furthermore, we derive explicit, geometry-aware utility bounds for both mechanisms.
In particular, for the noncompact setting we show that arbitrarily strong privacy can be achieved while maintaining finite (and explicitly controlled) utility under confining Langevin dynamics. Finally, as an application, we develop privacy-preserving estimation of generalized Fr\'echet means (including the classical Fr\'echet mean as a special case) \cite{afsari2011lp}, which is a fundamental notion of location for random objects underpinning broad areas of geometric statistics and learning \cite{bhattacharya2003intrinsic,petersen2019frechet,agueh2011barycenters,pennec2019riemannian}.
We provide sensitivity analysis beyond the squared-loss case and characterize regimes exhibiting phase-transition-type behavior in sensitivity/robustness.

\subsection{Organization}\label{subsec:organization}

Section~\ref{sec:preliminaries} introduces geometric and probabilistic preliminaries, including diffusion processes and the dimension-free Harnack inequality.
Section~\ref{sec:mainresults} presents the diffusion mechanisms and establishes their RDP and utility guarantees.
Section~\ref{sec:generalized} develops the application to privatized generalized Fr\'echet mean estimation.
Additional technical proofs are deferred to the supplementary material.

\section{Preliminaries}\label{sec:preliminaries}
\subsection{Background and Notations}
Let $\mathcal M$ be a smooth manifold of dimension $m$, 
equipped with a Riemannian metric $g$, that is, a smoothly varying inner product $g_p$ on each tangent space $T_p\mathcal M$. The distance between two points $p,q \in \mathcal M$ 
is defined as the infimum of the lengths of piecewise smooth curves 
connecting $p$ and $q$. 
This construction defines a metric $d$ on $\mathcal M$, called the Riemannian distance induced by $g$. 
The metric space $(\mathcal M,d)$ is complete by the Hopf-Rinow theorem if and only if $(\mathcal M,g)$ is geodesically complete.

A vector field $V$ is a smooth section of the tangent bundle 
$T\mathcal M$, assigning to each point $p \in \mathcal M$ a tangent 
vector $V(p) \in T_p\mathcal M$. Since tangent spaces at different 
points are distinct, comparing $V(p)$ and $V(q)$ requires additional 
structure. The Levi-Civita connection provides a canonical notion 
of differentiation along curves: for a smooth curve $\gamma(t)$, 
the covariant derivative of $V$ along $\gamma$ is defined by $D_t V = \nabla_{\dot\gamma(t)} V$.
A curve $\gamma(t)$ is called a geodesic if it satisfies
$\nabla_{\dot\gamma(t)} \dot\gamma(t) = 0$, which generalizes the notion of a straight line in Euclidean space.

Fixing a point $p \in \mathcal M$, the exponential map 
$\Exp_p : T_p\mathcal M \to \mathcal M$ is defined by
$\Exp_p(V) = \gamma(1)$,
where $\gamma(t)$ is the geodesic starting at $p$ with initial velocity $V$. 
Locally, $\Exp_p$ is a diffeomorphism near the origin since its 
differential at $0 \in T_p\mathcal M$ is the identity. However, in 
general it need not be globally injective. The largest radius $r>0$ 
such that $\Exp_p$ is a diffeomorphism on the ball 
$B_r(0) \subset T_p\mathcal M$ is called the injectivity radius at $p$, denoted by $\inj_p$. The injectivity radius of the manifold is denoted and defined as $\inj{\mathcal M} = \inf_{p \in \mathcal M} \inj_p$.
The cut locus of $p$, denoted $\Cut_p$, is the set of points where 
the exponential map ceases to be minimizing or fails to be a local 
diffeomorphism; equivalently, it consists of endpoints of maximal 
minimizing geodesics issuing from $p$.

There is a canonical volume form induced by the metric $g$, which we denote by
$d\mathrm{vol}$. With $g$ fixed, the Riemannian gradient of a smooth function $f$ on $\mathcal M$,
denoted by $\nabla f$, is characterized by the identity
$\langle \nabla f, X \rangle = X(f)$,
for any smooth vector field $X$. The divergence of a vector field $V$ is denoted by
$\diver(V)$.
Curvature is a key notion for describing the local geometry. The Riemann curvature
tensor is the $(1,3)$-tensor $
R: T\mathcal M \times T\mathcal M \times T\mathcal M \to T\mathcal M
$ defined by
\[
R(X,Y)Z
=
\nabla_X \nabla_Y Z
-
\nabla_Y \nabla_X Z
-
\nabla_{[X,Y]} Z,
\]
where $\nabla$ is the Levi--Civita connection associated with $g$. The Ricci curvature tensor is obtained by contracting the first and third indices:
\[
\mathrm{Ric}(Y,Z)
=
\operatorname{tr}\!\bigl(X \mapsto R(X,Y)Z\bigr),
\]
where $\operatorname{tr}$ denotes the trace. We write $\Ric_p(\cdot,\cdot)$ for the
Ricci tensor at $p\in\mathcal M$. When $Y=Z$, we use the shorthand
$\Ric_p(Y) =\Ric_p(Y,Y)$.
Heuristically, $\Ric_p$ can be viewed as an average of sectional curvatures through
a given direction. Geometrically, lower bounds on $\Ric$ control volume growth and
metric distortion, and they also constrain the global topology. For instance, the
following classical result shows that strictly positive Ricci curvature forces
compactness.

\begin{proposition}[Bonnet-Myers]\label{Thm:BonnetMyers}
Let $\mathcal M$ be a complete, connected Riemannian manifold of dimension $m$.
Assume there exists $r>0$ such that, for every $p\in\mathcal M$ and every unit
vector $V\in T_p\mathcal M$, we have $\Ric_p(V) \ge (m-1)/r^2$. Then any two points in $\mathcal M$ can be joined by a minimizing geodesic segment
of length at most $\pi r$. In particular,
$\mathrm{diam}(\mathcal M) \le \pi r$, and hence $\mathcal M$ is compact.
\end{proposition}

We will also see that Ricci curvature plays a central role in the behavior of
diffusion processes on Riemannian manifolds.

\subsection{Brownian Motion and the Laplace--Beltrami Operator}

Let $(\m,g)$ be a Riemannian manifold. For any smooth function $f:\m\to\R$, the Laplace--Beltrami operator is defined by $\Delta f = \diver(\nabla f)$. Under this sign convention, $\Delta$ is a negative and formally self-adjoint operator on $C^\infty(\m)$. More precisely, for any $f,g\in C^\infty(\m)$ with compact support, one has
\begin{align*}
    \int_{\m} f \,\Delta g \, d\mathrm{vol}(z)
    =
    \int_{\m} g \,\Delta f \, d\mathrm{vol}(z).
\end{align*}

The heat kernel is the fundamental solution to the heat equation associated with $\Delta$. We recall the following standard definition.

\begin{definition}
A smooth function $p:\m\times\m\times(0,\infty)\longrightarrow \mathbb{R}
$ is called a \emph{heat kernel} on $(\m,g)$ if it satisfies the following properties:
\begin{enumerate}
    \item[(i)] For each fixed $z\in\m$, the function $p(\cdot,z,t)$ solves the heat equation in the $x$-variable:
    \[
    \left(\frac{\partial}{\partial t}-\Delta_x\right)p(x,z,t)=0.
    \]
    \item[(ii)] Symmetry:
$p(x,z,t)=p(z,x,t)$, for any $x,z\in\m,\ t>0.$
 
    \item[(iii)] Initial condition: as $t\downarrow 0$, $p(x,z,t)\to \delta_z(x)$, in the sense of distributions, where $\delta_z$ denotes the Dirac mass at $z$.
\end{enumerate}
\end{definition}

It is well known that $\m$ is \emph{stochastically complete} if
\[
\int_{\m} p(x,z,t)\,d\mathrm{vol}(z)=1
\qquad\text{for all }x\in\m,\ t>0.
\]
Throughout this paper, we assume stochastic completeness. Under this assumption, the heat kernel is uniquely determined and defines the transition density of Brownian motion $(B_t)_{t\ge 0}$ on $\m$, started from $x$. In other words, for each $t>0$, the law of $B_t$ given $B_0=x$ admits density $p(x,z,t)$ with respect to the Riemannian volume measure. We refer to Section~11.4 of \cite{grigoryan2009heat} for background.

The Laplace--Beltrami operator and the heat kernel also encode important spectral information about the underlying manifold. In the compact case, the spectral theory is particularly transparent. Since $\Delta$ is self-adjoint and has compact resolvent on $L^2(\m)$, its spectrum is discrete and admits the following classical decomposition.

\begin{proposition}[Spectral decomposition]\label{prop:Sturm-Liouville}
Suppose that $\m$ is compact. Then there exists a complete orthonormal basis $\{\varphi_0,\varphi_1,\varphi_2,\dots\}$ of $L^2(\m)$ consisting of smooth eigenfunctions of $\Delta$, where
\[
\Delta \varphi_j = -\lambda_j \varphi_j,
\qquad
0=\lambda_0 \le \lambda_1 \le \lambda_2 \le \cdots,
\qquad
\lambda_j\to\infty.
\]
Moreover, the heat kernel admits the expansion
\begin{align}
    p(x,z,t)
    =
    \frac{1}{\mu(\m)}
    +
    \sum_{j=1}^{\infty} e^{-\lambda_j t}\varphi_j(x)\varphi_j(z), \qquad \text{where} \qquad \mu(\m)=\mathrm{vol}(\m).
\end{align}
\end{proposition}

Associated with the heat kernel is the heat propagation $(P_t)_{t\ge 0}$ on $L^2(\m)$, defined by
\begin{equation}
    P_t f(x)
    =
    \int_{\m} p(x,z,t)f(z)\,d\mathrm{vol}(z).
\end{equation}
The family $(P_t)_{t\ge 0}$ satisfies the semigroup property $P_{t+s}=P_t\circ P_s,$ for any $
 s,t\ge 0.$

We shall be particularly interested in both the short-time and long-time behavior of the heat kernel and the associated Brownian motion. At small times, the heat kernel exhibits a Gaussian-type asymptotic behavior governed by the Riemannian distance. More precisely, Varadhan's asymptotic formula asserts that
\begin{align*}
    \lim_{t\to 0_+} 4t\log p(x,z,t) = -d^2(x,z).
\end{align*}
Thus, in the short-time regime, the heat kernel behaves analogously to a Riemannian Gaussian density centered at the starting point.
\subsection{Langevin Process and Bakry-\'{E}mery Ricci Curvature}\label{subsec:2.3}
Let us consider the operator $L=\Delta-\langle\nabla V, \nabla \rangle$, where $V(x)$ is a potential function defined on $\m$. Let $X_t$ be the diffusion process generated by $L$. $X_t$ is a so-called Langevin process with drift and satisfies the following stochastic differential equation:
$$d X_t=\sqrt{2}d B_t-\nabla V(X_t)dt.$$ 
Similarly, there is also a Markov kernel corresponding to $L$ and we can define the drifted heat propagation $P_t$ in the same manner.
The Bochner formula \cite{wei2009comparison} of $L$ yields an additional term $\nabla^2V$ compared with the corresponding formula for the Laplacian-Beltrami operator. Thus we define the Bakry-\'{E}mery Ricci curvature tensor of $L$ to be $\Ric_V=\Ric+ \nabla^2 V$. This curvature tensor plays an important role in the diffusion process. In particular, the Langevin process $X_t$ asymptotically converges to a stationary distribution \cite{bakry2006diffusions,sturm2005convex} if $\Ric_V\ge K$ for some $K >0$ . The curvature condition also implies a Log-Sobolev inequality, which gives the exponential decay of the entropy function and Talagrand's transport inequality \cite{von2005transport}. Combining these results, we have the following exponential decay in $2$-Wasserstein distance:
\begin{align}\label{eq:expdecayinWass}
    W_2( P_t\nu, P_t\eta) \le e^{-Kt}W_2(\nu,\eta),
\end{align}
    where $P_t$ is the (drifted) heat propagator generated by $L$ and $\nu$ and $\eta$ are probability measures within Wasserstein space $\mathcal{P}_2$.
\begin{remark}
From a geometric viewpoint, the Langevin process can also be generated as Brownian motion with a weighted Laplacian-Beltrami operator. Specifically, we consider the weighted manifold $(\m,g,e^{-V}d\mathrm{vol}_g)$, which has the same topological space and metric as the original one but with a weighted volume form. Then the Langevin process $X_t$ is just the natural Brownian motion on the weighted manifold. Meanwhile, the Bakry-\'{E}mery Ricci curvature tensor plays the same role as the standard Ricci curvature. 
    \end{remark}

\subsection{Dimension-Free Harnack Inequality}
The classical Harnack inequality was initially established for harmonic functions supported on a closed ball in Euclidean space. Later, Li and Yau extended this result to Riemannian manifolds, providing a Harnack inequality for the heat kernel \cite{li1986parabolic}. This celebrated work has found widespread applications in geometric analysis, including heat kernel estimates and logarithmic Sobolev inequalities. Moreover, it has inspired analogous inequalities in the context of Ricci flow, which play a crucial role in the proof of the Poincar\'{e} conjecture \cite{perelman2002entropy,perelman2003finite,perelman2003ricci}. In this section, we introduce the dimension-free Harnack inequality \cite[Theorem~1.2]{wang2004equivalence}. 
\begin{proposition}
 Let $\m$ be an $m$-dimensional complete Riemannian manifold and $L=\Delta-\langle \nabla V, \nabla\rangle$. Then the following two conditions are equivalent:
 \begin{enumerate}
        \item[(i)] The Bakry-\'{E}mery Ricci curvature satisfies $\Ric_V(X)\ge -K\|X\|^2$.
     \item[(ii)] For every $\alpha >1$ the following Harnack inequality is satisfied:
     \begin{equation}\label{eq:dimensionfreeHarnack}
    (P_t|f(x)|)^\alpha \le P_t|f|^\alpha(y)\exp\bigg[ \frac{K\alpha d^2(x,y)}{2(\alpha-1)(1-e^{-2Kt})}\bigg],     
     \end{equation}
     where $f$ is any measurable function on $\m$ and $P_t$ is the (drifted) heat propagation corresponding to $L$.
    
 \end{enumerate}
\end{proposition}
\begin{remark}
A notable distinction from Li–Yau’s version is that the inequality does not depend on the dimension of the underlying manifold. This property makes it applicable to the study of functional operators in infinite-dimensional spaces and is therefore referred to as the dimension-free Harnack inequality. Moreover, the inequality is optimal for all $t$, since it is equivalent to the Bakry-\'{E}mery Ricci curvature condition. 
\end{remark}

% =======================
\subsection{Stochastic Processes, Martingales, and It\^{o} Integrals}
\label{subsec:stoch_prelim}
% =======================

This subsection collects the key results from stochastic calculus that we will use repeatedly in later
proofs. For a detailed and rigorous development of stochastic analysis on Riemannian
manifolds, we refer to \cite[Chapters~1--4]{hsustochastic}; see also the discussion of Brownian
motion and heat kernels therein.

Throughout the paper we use the Laplace--Beltrami operator $\Delta=\mathrm{div}(\nabla)$
as in Section~2.2, and the heat kernel $p(x,z,t)$ is the fundamental solution
to the heat equation with this choice of $\Delta$.
Under this convention, the diffusion term corresponding to $\Delta$ carries a factor $\sqrt{2}$
when written in terms of the standard Brownian motion $W$ (this is the same convention used
in the Langevin SDE displayed in Section~2.3).

A real-valued adapted process $(M_t)_{t\ge 0}$ is a \emph{martingale} if
$\mathbb{E}|M_t|<\infty$ for all $t$ and
$\mathbb{E}[M_t\mid \mathcal{F}_s]=M_s$ almost surely for all $0\le s\le t$.
It is a \emph{local martingale} if there exists an increasing sequence of stopping times
$\tau_n\uparrow\infty$ such that each stopped process $M^{\tau_n}_t:=M_{t\wedge\tau_n}$
is a martingale.

\begin{proposition}[Mean of a martingale]
\label{prop:martingale_mean}
If $(M_t)_{t\ge 0}$ is a martingale, then $\mathbb{E}[M_t]=\mathbb{E}[M_0]$ for every $t\ge 0$.
In particular, if $M_0=0$, then $\mathbb{E}[M_t]=0$ for all $t$.
\end{proposition}

\begin{proposition}[From local martingale to martingale via $L^2$]
\label{prop:local_to_martingale}
Let $(M_t)_{t\ge 0}$ be a continuous local martingale with $M_0=0$.
If $\mathbb{E}[M_t^2]<\infty$ for each $t$ (equivalently, $\mathbb{E}[ [M]_t ]<\infty$),
then $(M_t)_{t\ge 0}$ is in fact a (square-integrable) martingale and hence $\mathbb{E}[M_t]=0$.
\end{proposition}

Let $L$ be a second-order differential operator on $\m$ of the form $L \;=\; \Delta - \langle \nabla V, \nabla \rangle$,
as in Section~2.3.
A continuous $\m$-valued process $(X_t)_{t\ge 0}$ is called an $L$-\emph{diffusion} if for every
$f\in C^\infty(\m)$ the real-valued process
\[
M_t^f \;:=\; f(X_t)-f(X_0)-\int_0^t (Lf)(X_s)\,ds
\]
is a local martingale. This is the standard ``martingale problem'' formulation; see
\cite[Chapter~3]{hsustochastic}.

\begin{proposition}[It\^{o} formula for an $L$-diffusion {\cite[Chapters~1 and~3]{hsustochastic}}]
\label{prop:ito_generator_form}
Let $(X_t)_{t\ge 0}$ be an $L$-diffusion on $\m$, and let $f\in C^\infty(\m)$.
Then
\[
f(X_t)
=
f(X_0)
+\int_0^t (Lf)(X_s)\,ds
+M_t^f,
\]
where $M^f$ is a continuous local martingale.
Moreover, the quadratic covariation satisfies
\[
[M^f,M^g]_t
=
2\int_0^t \langle \nabla f(X_s), \nabla g(X_s)\rangle\,ds,
\qquad f,g\in C^\infty(M),
\]
and in particular $[M^f]_t
=
2\int_0^t \|\nabla f(X_s)\|^2\,ds.$
Consequently, if $\mathbb{E}\int_0^T \|\nabla f(X_s)\|^2 ds<\infty$ for a given $T>0$
(e.g.\ when $\m$ is compact and $f$ is smooth, so $\|\nabla f\|_\infty<\infty$),
then $M^f$ is a true martingale on $[0,T]$ and
\[
\mathbb{E}[f(X_t)]
=
f(X_0)+\mathbb{E}\int_0^t (Lf)(X_s)\,ds,
\qquad 0\le t\le T.
\]
\end{proposition}
In order to analyze the stochastic process on Riemannian manifolds, we consider the radial process and introduce the Laplacian comparison theorem.
\begin{definition}
    Let $(\m,d)$ be a Riemannian manifold and $X_t$ be a diffusion process on $\m$. We define the radial process $r_t:=d(X_0,X_t)$, where $X_0$ is the starting point of $X_t$. Sometimes we also denote it by $r(X_t)$ or write simply as $r$. 
\end{definition}
\begin{remark}[On nonsmooth test functions and the cut locus]
\label{rem:cutlocus_ito}
On a general manifold, functions such as $x\mapsto d(x,o)$ are not smooth on the cut locus,
and one needs a generalized It\^{o} formula involving local time terms; see
\cite[Chapter~3, especially the discussion around the radial process decomposition]{hsustochastic}.
In Section~3.3 we restrict to Hadamard manifolds, in which the cut locus is empty; in that case
the squared distance $x\mapsto d(x,o)^2$ is smooth on $\m$, and the standard It\^{o} formula
(Proposition~\ref{prop:ito_generator_form}) applies directly to such distance-based Lyapunov
functions.
\end{remark}
Now if we restrict the process out of the cut locus of $\m$, we have the well-known estimation of the radial process.
\begin{proposition}[Laplacian Comparison Theorem]
   Let $\m$ be an $m$-dimensional Riemannian manifold. Suppose that on a geodesic ball $B(R)$ away from the cut locus the sectional curvature is bounded from above by $\kappa>0$ and the Ricci curvature is bounded from below by $-(m-1)K$ with $K > 0$. Then
\[
(m-1)\sqrt{\kappa}\cot(\sqrt{\kappa} r)\ \le\ \Delta r\ \le\ (m-1)\sqrt{K}\coth(\sqrt{K} r)
\]
on $B(R)$.
Furthermore, in the limit $K \rightarrow 0$, we obtain the bound $\Delta r \le (m-1)/r$.
\end{proposition}

\section{Main Results} \label{sec:mainresults}
\subsection{R\'{e}nyi Differential Privacy}

Informally, differential privacy controls the sensitivity of the output distribution of a randomized algorithm to small perturbations in its input. The $\varepsilon$-differential privacy gives the multiplicative upper bound on the worst-case change in the output \cite{dwork2006differential}.
\begin{definition}[$\varepsilon$-Differential Privacy]
A randomized algorithm $\mathcal{F} : \mathcal{D} \to \mathcal{Y}$ is said to satisfy
$\varepsilon$-differential privacy if for any two adjacent datasets
$D, D' \in \mathcal{D}$ that differ in at most one individual, and for any
measurable subset $S \subseteq \mathcal{Y}$, it holds that
\[
\mathbb{P}\bigl( \mathcal{F}(D) \in S \bigr)
\;\le\;
e^{\varepsilon}
\,\mathbb{P}\bigl( \mathcal{F}(D') \in S \bigr).
\]
\end{definition}
Many DP mechanisms are available, each offering advantages in different scenarios. In this paper, we focus on R\'{e}nyi differential privacy \cite{mironov2017renyi}, which bounds the perturbation error through R\'{e}nyi divergence.
\begin{definition}[Rényi Differential Privacy (RDP)]
Let $\mathcal{F} : \mathcal{D} \to \mathcal{Y}$ be a randomized mechanism. 
For any two adjacent datasets $D, D' \in \mathcal{D}$ that differ in at most one entry, 
$F$ is said to satisfy \emph{$(\alpha, \varepsilon)$-Rényi Differential Privacy (RDP)} 
if, for fixed $\alpha > 1$, it holds that
\[
D_{\alpha}\bigl(\mathcal{F}(D)\,\|\,\mathcal{F}(D')\bigr) \le \varepsilon,
\]
where $D_{\alpha}(P\|Q)$ denotes the Rényi divergence of order $\alpha$ between two distributions $P$ and $Q$, defined as
\begin{equation}
D_{\alpha}(P\|Q)
= \frac{1}{\alpha - 1}\log 
\mathbb{E}_{x \sim Q}
\!\left[
\left(\frac{p(x)}{q(x)}\right)^{\alpha}
\right],
\end{equation}
where $p(x),\ q(x)$ denote the probability density functions of $P$ and $Q$ with respect to a reference measure. 
\end{definition}
Note that order $1$ R\'{e}nyi divergence is defined by taking the limit of $D_\alpha$ as $\alpha$ approaches $1$ from above. R\'{e}nyi divergence enjoys two advantages in our scenario. First, it is highly compatible with Brownian motion and reflects, in a principled way, how the underlying geometry influences the privacy budget. Moreover, it satisfies the following Data Processing Inequality; see \cite[Theorem~9]{vanerven2014renyi}.

\begin{proposition}[Data Processing Inequality.]\label{Prop:dataprocessing}
    Let $\Phi$ be a measurable map from $M$ to $N$. For two distributions $P$ and $Q$ defined on $M$, we have 
    \begin{equation*}
        D_\alpha(P \Vert Q) \ge D_\alpha (\Phi_\#P\Vert \Phi_{\#}Q),
    \end{equation*}
    where $\Phi_{\#}P$ and $\Phi_{\#}Q$ are push-forward distributions of $P$ and $Q$, respectively.
\end{proposition}

\subsection{Achieving RDP by Heat Diffusion}
In this section, we introduce the Brownian Motion (BM) mechanism. The idea is straightforward from a physical viewpoint: we allow the input data to evolve according to Brownian motion on a Riemannian manifold for a prescribed time 
$t$, and release the resulting point as the privatized output. The Brownian motion, which can also be viewed as a geodesic random walk in the infinitesimal-step limit, injects noise whose magnitude is controlled by the diffusion time $t$. We emphasize that this mechanism is intrinsic and naturally respects the geometry. The remaining part of this section is mainly concerned with rigorous mathematical analysis of the privacy budget. Before proceeding, we firstly introduce the following concept of differential privacy.  
We define the summary $f$ to be a map from $\mathcal{D}$ to the manifold $\m$. That is, for any $D\in \mathcal{D}$, $f(D)$ is a point on the Riemannian manifold. One can choose the summary to be some statistic of the dataset $D$ that is manifold-valued. We denote the BM mechanism with starting point $f(D)$ and diffusion time $t$ by $B_t(f(D))$. In the DP context, the summary $f$ is always assumed to have a global sensitivity ($\Delta$) with respect to the Riemannian distance. To be specific, given two adjacent datasets $D$ and $D'$ that differ at most by one individual, we have $d(f(D),f(D'))<\Delta$. Under the above assumption, the following theorem gives the explicit privacy budget.

\begin{theorem}\label{Thm:RDP1}
    Let $\m$ be an $m$-dimensional Riemannian manifold such that $\Ric \ge -K$ for some $K$ and let $f$ be an $\m$-valued summary with global sensitivity $\Delta$. Then the BM mechanism with starting point $f(D)$ and diffusion time $t$ satisfies $(\alpha ,\varepsilon)$-RDP with $\varepsilon=\frac{K \alpha \Delta^2}{2(1-e^{-2Kt})}.$ 
\end{theorem}

An immediate observation is that if we differentiate $\varepsilon$ with respect to $t$, we obtain $$\frac{d\varepsilon}{dt} =\frac{-K^2\alpha \Delta^2e^{-2Kt}}{(1-e^{-2Kt})^2} \le 0.$$
This is independent of the value of $K$. Note that with larger diffusion time $t$, the mechanism certainly has a smaller privacy budget and hence stronger privacy. This aligns with our intuition about the Brownian motion mechanism mentioned before. The geometry of the underlying manifold also has a significant influence on the privacy budget. $K$ that parametrizes the Ricci curvature lower bound, determines the asymptotic behavior of the privacy budget with respect to $t$. This matters from a differential-privacy perspective since we expect that our mechanism can achieve an arbitrarily strong privacy level with sufficiently large $t$. We roughly classify the Riemannian manifolds into three categories based on the sign of the Ricci curvature and tackle this problem:

\begin{enumerate}
    \item[(i)] In the case that $\Ric >0$ everywhere, the manifold is compact (e.g., a sphere, Stiefel manifold, Grassmannian) by Proposition~\ref{Thm:BonnetMyers}. Under the convention of Theorem~3.3, we have $K<0$. We may substitute $N=-K$ and we have $\varepsilon=N\alpha \Delta^2/(2e^{2Nt}-2)$ for $N >0$. It is clear that the privacy budget will exponentially decay to $0$ with increasing $t$. This property guarantees that we can achieve an arbitrarily small privacy budget with sufficiently large diffusion time.
    \item[(ii)] In the case where the underlying manifold is Euclidean, $\Ric=0$ everywhere. Now we have $K=0$, which is a simple pole in the sense that the denominator of $\varepsilon$ equals to $0$. If we consider the Taylor expansion of $e^{-2Kt}$ and let $K$ tend to $0$, we obtain 
\begin{align*}
   \lim_{K\rightarrow0} \varepsilon=\lim_{K\rightarrow 0}\frac{K\alpha \Delta^2}{2(1-(1-2Kt+\CO(K^2)))}=\frac{\alpha \Delta^2}{4t}.
\end{align*}
In Euclidean space, the transition density functions of $B_t(x)$ and $B_t(y)$ are normal distributions $\mathcal{N}(x,2tI_m)$ and $\mathcal{N}(y,2tI_m)$, respectively. One can compute the R\'{e}nyi divergence of two normal distributions \cite[Equation~10]{vanerven2014renyi}, which is exactly the above limit. Thus, for Euclidean space, our upper bound is sharp.
\item[(iii)] If the manifold has negative Ricci curvature everywhere, the parameter $K$ is a positive number. In this case, we increase the diffusion time to infinity and find that the limit is given by $K\alpha \Delta^2/2$. Since $\varepsilon$ is decreasing with respect to $t$, this limit actually serves as a lower bound for $\varepsilon$. This demonstrates a "pathology" of the pure BM mechanism when the underlying manifold has negative Ricci curvature: the privacy guarantee cannot be improved beyond $\varepsilon = K\alpha \Delta^2/2$, regardless of how large the diffusion time 
$t$ becomes. 
\end{enumerate}

In summary, the BM mechanism works  well when $\m$ has a non-negative Ricci curvature lower bound. In particular, it recovers the Gaussian mechanism in the sense of the RDP definition when the manifold is Euclidean. The following theorem underscores the ability of the BM mechanism to achieve R\'{e}nyi differential privacy on certain manifolds.  

\begin{theorem}[BM Mechanism]\label{Thm:BMM}
   Let $\m$ be a complete Riemannian manifold such that $\Ric(X) \ge -K \Vert X \Vert^2$ for some $K < 0$ and let $f$ be an $\m$-valued summary with global sensitivity $\Delta$. Then, the BM mechanism with starting point $f(D)$ satisfies $(\alpha, \varepsilon)$-RDP with $$t=\frac{\log(2\varepsilon)-\log(2\varepsilon-K\alpha \Delta^2)}{2K}.$$ In particular, when $\m$ is a Euclidean space, the BM mechanism satisfies $(\alpha, \varepsilon)$-RDP with $t=\frac{\alpha\Delta^2}{4\varepsilon}$, which coincides with the result in \cite{mironov2017renyi}.
\end{theorem}
Theorem~\ref{Thm:BMM} follows directly from Theorem~\ref{Thm:RDP1}. To address the pathology mentioned above, we must analyze the asymptotic behavior of Brownian motion on certain manifolds. The key issue is whether the Brownian motion can ``forget'' its starting point and become ergodic in the long-time limit. This is closely tied to DP: if our BM mechanism injects noise that becomes asymptotically independent of the original data, then the privacy budget can be arbitrarily small. Conversely, if the injected noise remains dependent of the original data regardless of how large the diffusion time is, then the privacy budget must be bounded below. This ``forgetfulness'' phenomenon is determined by the sign of the Ricci curvature of the underlying manifold. Next, we formalize this observation in rigorous mathematical terms.

 Let us start with Brownian motion on a positively curved Riemannian manifold. In this setting, the Laplacian-Beltrami operator $\Delta$ is compact, negative, and self-adjoint with a discrete spectrum of eigenvalues. By Proposition~\ref{prop:Sturm-Liouville}, the heat kernel $p(x,z,t)$ admits the following decomposition as the transition kernel of Brownian motion:
\begin{align*}
    p(x,z,t)=\frac{1}{\mu(\m)}+\sum_{j=1}^{\infty}e^{-\lambda_j t}\varphi_j(x)\varphi_j(z).  
\end{align*}
Clearly, as $t\to\infty$, $p(x,z,t)$ converges to the uniform distribution supported on $\m$. The R\'{e}nyi divergence between two uniform distributions is $0$, which explains why the privacy budget converges to $0$. Practically speaking, the Brownian motion $B_t$ gradually ``forgets'' its starting point and may end up at any location on the manifold with equal probability. This is a favorable phenomenon from a DP perspective. As mentioned in Section~\ref{sec:preliminaries}, for small diffusion time $t$ the transition density admits the classical Gaussian-type short-time expansion (locally uniformly for $z$ in a normal neighborhood of $x$),
\[
p(x,z,t)
=
(4\pi t)^{-m/2}\exp\!\Big(-\frac{d(x,z)^2}{4t}\Big)\big(u_0(x,z)+O(t)\big),
\qquad t\downarrow 0,
\]
where $u_0$ is smooth with $u_0(x,x)=1$; see, e.g., \cite{grigoryan2009heat,rosenberg1997laplacian}.
Thus, in normal coordinates, the law of $B_t(x)$ resembles a Gaussian distribution centered at $x$ for small $t$, while for large $t$ it approaches the uniform distribution on a compact manifold. In this sense, $B_t$ continuously interpolates between these two regimes. As $t$ increases, the injected noise becomes progressively less dependent on the original data, thereby yielding stronger privacy guarantees.

\begin{figure}[h!]
    \centering
    \stackunder[5pt]{\includegraphics[width=0.29\linewidth]{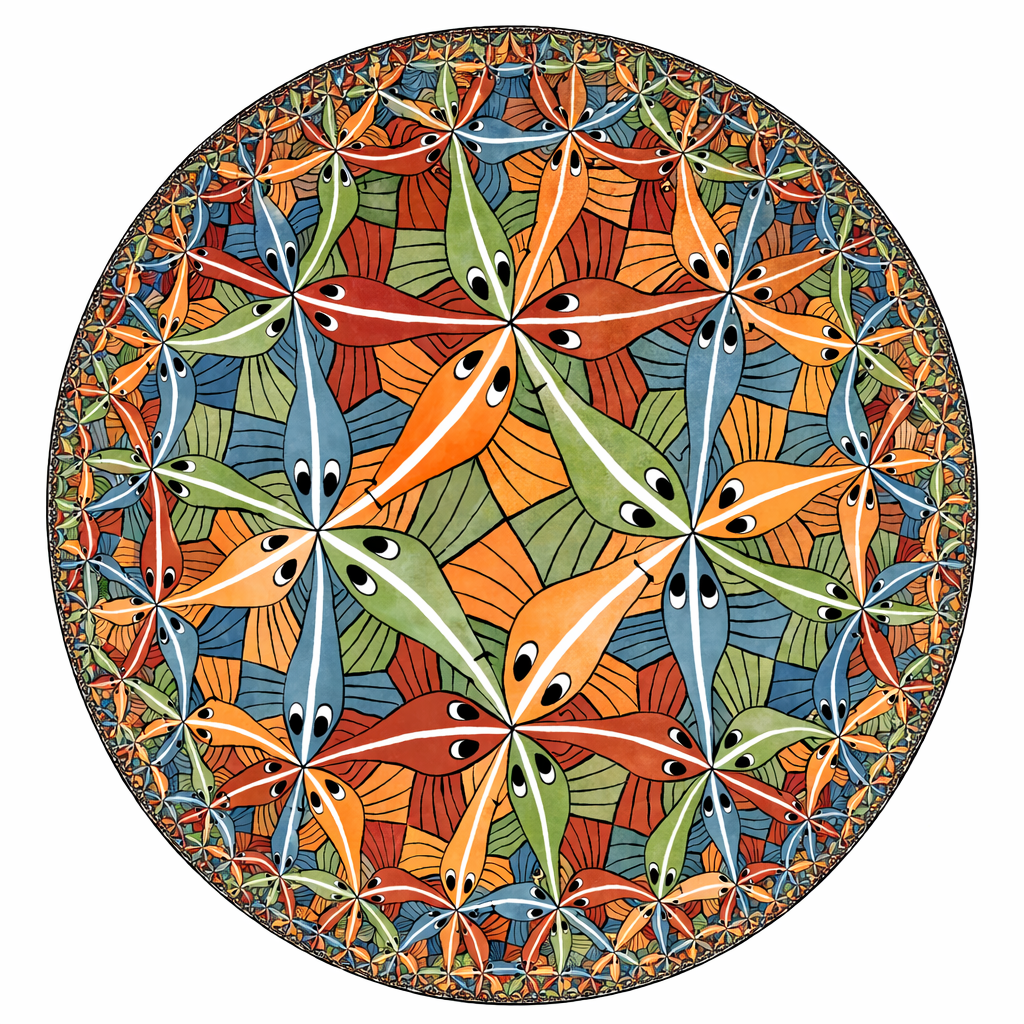}}{}
   \hspace{0.02\linewidth}
    \stackunder[5pt]{\includegraphics[width=0.32\linewidth]{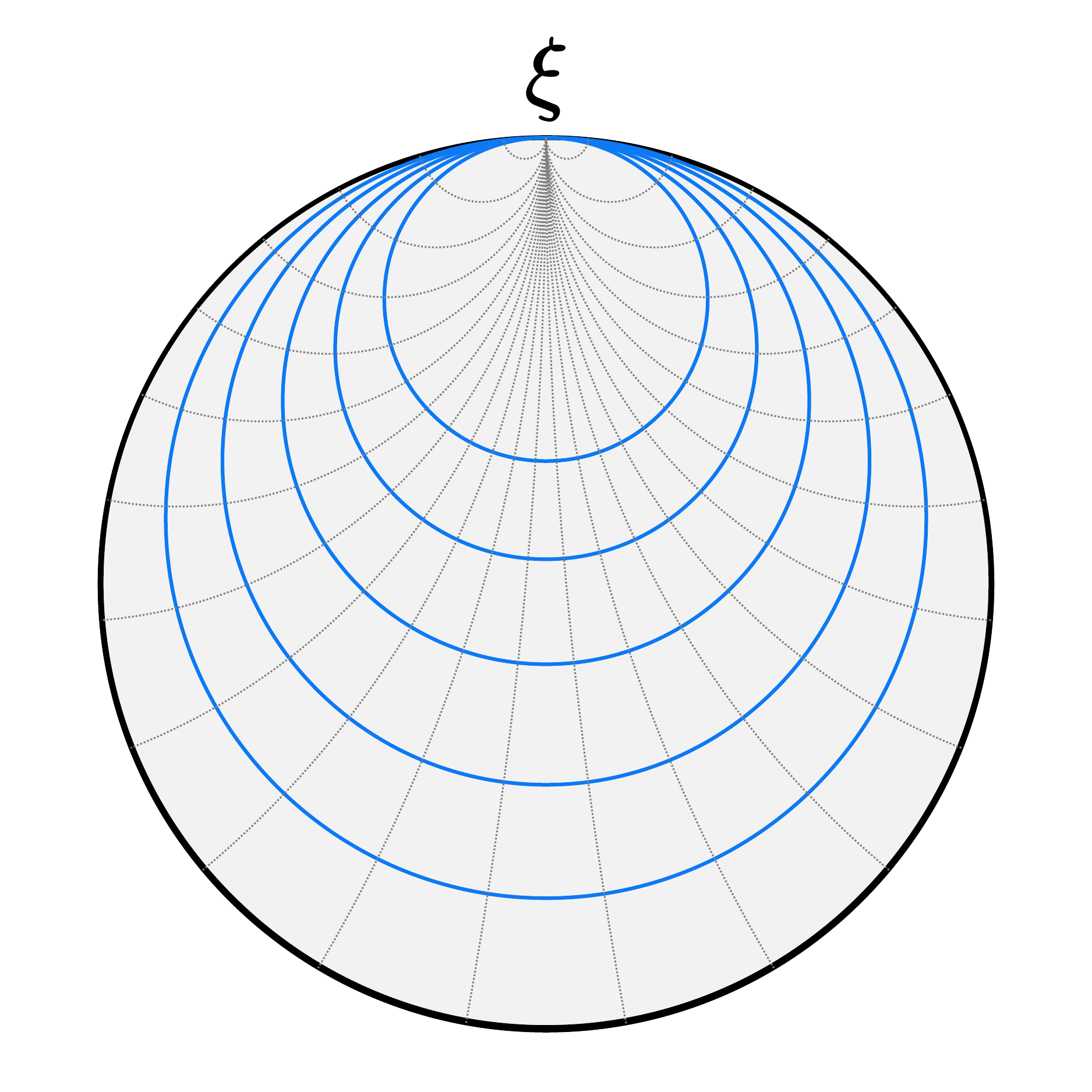}}{}
    \caption{The left picture depicts Circle Limit III by Dutch artist M. C. Escher, which was inspired by the Poincar\'{e} disk model of hyperbolic space. The right picture visualizes the Gromov boundary of hyperbolic space in the Poincar\'{e} disk model. The black circle serves as $\partial \overline{\m}$. Blue circles are horospheres centered at $\xi \in \partial \overline{\m}$.
    \label{fig:placeholder}}
\end{figure}
However, for manifolds with bounded negative sectional curvature, the long-time behavior of
$B_t$ is qualitatively different: intuitively, geodesics diverge and Brownian trajectories
tend to drift toward the boundary at infinity.
In particular, consider a Hadamard manifold, i.e., a complete, simply connected manifold with non-positive sectional
curvature. Such $\m$ admits a visual compactification
$\overline{\m}=\m\cup\partial\overline{\m}$, where $\partial\overline{\m}=\overline{\m}\setminus \m$
is the Gromov boundary (homeomorphic to $\mathbb{S}^{m-1}$ under pinched negative curvature).
By \cite{sullivan1983dirichlet}, Brownian motion $B_t$ starting at any $x\in\m$ converges almost surely
to a random limit in $\partial\overline{\m}$. For $x\in\m$, one may define the probability
measure $\nu_x$ on $\partial\overline{\m}$.

\begin{definition}[Harmonic Measure Class]
Let $A \subset \partial \overline{\m}$ be a Borel set, $\nu_x(A)$ denote the probability that Brownian motion starting at $x$ converges to a limit in $A$.       
\end{definition}
Fixing $A$, Sullivan \cite{sullivan1983dirichlet} proves that $x\rightarrow\nu_x(A)$ is a non-trivial bounded harmonic function, suggesting that $\nu_x \ne \nu_y$ with $x\ne y$. We then consider the visual projection $\Phi$: $\m \rightarrow \partial\overline{\m}$ by sending $x$ to the geodesic ray $\gamma_o(x)$ with a fixed point $o$. Under this construction, we have $\Phi_{\#}(B_t(x))=\nu_x$. By Proposition~\ref{Prop:dataprocessing}, 
\begin{align*}
    D_{\alpha}(B_t(x)\Vert B_t(y)) \ge D_\alpha(\Phi_{\#}B_t(x)\Vert \Phi_{\#}B_t(y))=D_{\alpha}(\nu_x\Vert \nu_y)>0.
\end{align*}
This explains why the pure BM mechanism exhibits a certain "pathology" on negatively curved spaces. It also leads to the following question: 
\begin{tcolorbox}[colback=white]
Can we modify pure Brownian motion so that it asymptotically forgets the initial point on negatively curved manifolds?
\end{tcolorbox}

It turns out that one can introduce a suitable drift to Brownian motion so that the resulting stationary density no longer depends on the initial point; this modified process is known as the Langevin process.

\subsection{Langevin Process Mechanism}
In this section, we restrict the underlying space to be a Hadamard manifold (e.g., hyperbolic spaces). This assumption guarantees an empty cut locus for any point, which simplifies both the analysis and the computations, while still covering many important applications. 

To construct a Langevin process $X_t$ with a stationary density, the key step is to choose a suitable potential function $V$ such that the Bakry--\'{E}mery Ricci curvature tensor is positive (see Section~\ref{subsec:2.3}). A convenient choice is $V(x)=\lambda d^2(o,x)/2$ with parameter $\lambda$ and an arbitrary fixed point $o$. Intuitively, one may view this as placing an anchor at $o$ and pulling a pure Brownian motion toward the anchor. The parameter $\lambda$ controls the strength of this pull: larger $\lambda$ corresponds to stronger attraction. Our Langevin mechanism evolves the data via Langevin diffusion for time $t$. We denote the Langevin mechanism with starting point $a$ and anchor $o$ by $X_t^{o}(a)$. With a suitable choice of $\lambda$, the resulting Langevin mechanism satisfies the desired RDP guarantee.

\begin{theorem}\label{Thm:RDP2}
    Let $\m$ be a $m$-dimensional Hadamard manifold such that $\Ric(X)\ge -K\Vert X \Vert^2$ for some $K >0$. For a summary with bounded sensitivity $\Delta$, and when $\lambda >K$, the Langevin mechanism with anchor $o$ satisfies $(\alpha,\varepsilon)$-RDP with $t$ given by 
    \begin{align*}
    t=\frac{\log(1+(\lambda-K)\alpha\Delta^2/(2\varepsilon))}{2(\lambda-K)}.
\end{align*}
\end{theorem}

One may note that $\lambda$ can be chosen arbitrarily as long as $\lambda > K$. A natural question is how the values of $\lambda$ and $t$ influence our Langevin mechanism. As discussed in Section~\ref{subsec:2.3}, $\lambda$ directly affects the convergence rate of $X_t$ to its stationary distribution. Dynamically, a larger $\lambda$ leads to faster ``forgetting'' of the starting point $a$ and stronger contraction toward the anchor $o$. On the other hand, $t$ controls how close the released point remains to the input. In practice, one can calibrate these two parameters jointly with respect to utility.

\subsection{Analysis on Utility}
In this section, we analyze the utility of the heat diffusion mechanism and the Langevin mechanism. That is, we study the expected distance between the original data and the privatized output. Our strategy is to consider the associated radial process $r(X_t)$. This reduces the problem to a one-dimensional stochastic process on $\mathbb{R}$, where a variety of tools from stochastic analysis can be applied.

First, consider the case where the Riemannian manifold has a positive lower bound on the Ricci curvature. In this case, our RDP mechanism is the standard Brownian motion generated by $\Delta$. By Proposition~\ref{Thm:BonnetMyers}, $\m$ is compact, and the presence of a cut locus implies that the distance function is not smooth everywhere. Nevertheless, we have the following representation; see \cite[Theorem~3.5.1]{hsustochastic}:
\begin{equation}\label{eq:radialrepresentation}
    r(B_t)=\sqrt{2}\beta_t+\int_0^t\Delta r(X_s)\,ds -L_t,
\end{equation}
where $\beta_t$ is a one-dimensional Euclidean Brownian motion and $L_t$ is a non-decreasing stochastic process that increases only when $B_t$ lies in the cut locus of $o$. We then apply It\^{o}'s formula together with the Laplacian comparison theorem to obtain the desired result.

\begin{theorem}\label{Thm:boundofd}
    Assume $\m$ is a complete $m$-dimensional Riemannian manifold such that $\Ric(X)\ge K \Vert X\Vert^2$ for some $K>0$, and let $B_t$ be the Brownian motion starting at point $a$ with diffusion time $t$. Then we have 
    \begin{align*}
        \mathbb{E}[d(a,B_t)]\le \sqrt{2mt}.
    \end{align*}
\end{theorem}

Next, we consider the Hadamard manifold scenario. In this case, the underlying manifold is non-compact, which makes square integrability nontrivial. However, this property is crucial for showing that certain It\^{o} integrals are martingales. We therefore first establish the following lemma.

\begin{lemma}\label{lemma:squreintegrability}
    Let $(M,g)$ be a Hadamard manifold and consider the Langevin diffusion $(X_t)_{t\ge0}$ with generator $L=\Delta-\langle \nabla V,\nabla\rangle $, and $V(x)=\frac\lambda 2 d(o,x)^2$. Assume $X_t$ is non-explosive and that the Bakry-\'Emery curvature satisfies
$\mathrm{Ric}_V:=\mathrm{Ric}+\nabla^2 V \ge K$ for some $K>0$.
Fix $a=X_0$ and define $f(x)=d(a,x)^2$. We have for every $T<\infty$,
\[
\mathbb E\int_0^T \|\nabla f(X_s)\|^2\,ds <\infty.
\]
\end{lemma}

Combining Laplacian comparison and the cosine law in Hadamard manifolds with Lemma~\ref{lemma:squreintegrability}, we have certain estimation of utility.
\begin{theorem}\label{Thm:utility}
     Let $\m$ be a Hadamard manifold of dimension $m$ such that  $  \Ric(X)\ge -K \Vert X\Vert^2$ for some $K \ge 0$. Let $X_t$ be the Langevin diffusion process generated by $L=\Delta-\langle \nabla V, \nabla \rangle$ with starting point $a$ and $V(x)=\frac \lambda2 d^2(o,x)$. The point $o$ is chosen arbitrarily from $\m$. We have 
     \begin{align*}
         \mathbb{E}[d(a,X_t)] \le \Bigg( \frac{2m}{\lambda}+\bigg(d(o,a)+\frac{\sqrt{(m-1)K}}{\lambda}\bigg)^2\Bigg)^{1/2}\big(1-e^{-\lambda t}\big)^{1/2}. 
     \end{align*}
\end{theorem}

\begin{remark}[Choosing the anchor in the Langevin mechanism]
For Hadamard manifolds, the anchor $o$ should be chosen independently of the confidential
dataset.  If prior public information indicates that the data lie in a geodesic ball
$B(o,r)$, then the center $o$ is the natural choice.  In the absence of such information, one
may fix a deterministic reference point in advance.  Once $o$ is chosen, the parameter
$\lambda$ controls the strength of the pull toward $o$, while $t$ controls the total amount of
diffusion.  
\end{remark}

%------------------------------
% Calibration of confining Langevin parameters via RDP budget
%------------------------------
\begin{comment}
The utility-driven calibration developed in \eqref{eq:opt_1d} can be used to select
$\lambda$ and $t$ jointly.
{\color{red}  Fix $o,a\in\m$, we propose a calibration method with respect to minimizing the utility. 
Let $\alpha>1$ and an RDP budget $\varepsilon>0$. We define $$R(\lambda,t):=\frac{k\,\alpha\,\Delta^2}{2\big(e^{2kt}-1\big)},\qquad k:=\lambda-K.$$
To enforce $(\alpha,\varepsilon)$-RDP while minimizing the 
utility, we solve the constrained problem
\begin{equation}\label{eq:opt_constrained}
\begin{aligned}
\min_{\lambda,t}\quad & 
\Big( \frac{2m}{\lambda}+2d^2(o,a)+\frac{(m-1)K}{\lambda^2}\Big)^{1/2}
\Big(1-e^{-\lambda t}\Big)^{1/2} \\
\text{s.t.}\quad &
R(\lambda,t)= \varepsilon, \qquad\lambda>K \ \ (\Leftrightarrow k>0),\qquad t>0.
\end{aligned}
\end{equation}

 Solving $R(\lambda,t)=\varepsilon$ yields $t_{\mathrm{priv}}(\lambda)
=\frac{1}{2k}\log\!\Big(1+\frac{k\alpha\Delta^2}{2\varepsilon}\Big)$. Substituting $t=t_{\mathrm{priv}}(\lambda)$ into~\eqref{eq:opt_constrained}, the calibration reduces
to the one-dimensional optimization
\begin{equation}\label{eq:opt_1d}
\lambda^\star\in\argmin_{\lambda>(m-1)K}\ f(\lambda),
\qquad
f(\lambda)
:=
\Big( \frac{2m}{\lambda}+2d^2(o,a)+\frac{(m-1)K}{\lambda^2}\Big)^{1/2}
\Big(1-e^{-\lambda t_{\mathrm{priv}}(\lambda)}\Big)^{1/2},
\end{equation}
and the optimal time is then $t^\star=t_{\mathrm{priv}}(\lambda^\star)$.

}
\end{comment}

\section{Application to RDP Estimation of Generalized Fr\'echet Means}\label{sec:generalized}

Generalized Fr\'echet means extend the ordinary Euclidean mean to nonlinear spaces and provide canonical notions of location for random objects.
Since Fr\'echet's original distance-based formulation \cite{frechet1948random}, Fr\'echet means have become central to manifold statistics, both as population descriptors and as targets of large-sample inference for intrinsic and extrinsic means \cite{bhattacharya2003intrinsic,bhattacharya2005intrinsic2}.
Their generalized versions, such as Riemannian $L^p$ centers of mass and more general $\phi$-means, unify means, robust centers, and other barycenter-type summaries under a common distance-based framework \cite{afsari2011lp,aveni2024uniform}.
In recent years, Fr\'echet-type means have also become key ingredients in random-object regression and broader metric-space inference \cite{petersen2019frechet,lin2021tvfrechet,chen2022uniform,dubeymuller2019fanova}, in tree- and forest-based learning for non-Euclidean responses \cite{capitaine2024frechetrf}, in optimal transport through Wasserstein barycenters \cite{agueh2011barycenters}, and even in differentiable geometric layers for machine learning \cite{lou2020frechet}.
For these reasons, developing privacy-preserving procedures for generalized Fr\'echet means is of independent statistical and practical importance.

We now specialize the diffusion mechanisms of Section~\ref{sec:mainresults} to the
release of the generalized Fr\'echet mean. 
Thus, the privacy analysis reduces to bounding the global sensitivity of the
generalized Fr\'{e}chet mean and then substituting this bound into the diffusion-based RDP results of
Theorems~\ref{Thm:BMM} and~\ref{Thm:RDP2}.
\begin{definition}[Empirical Fr\'echet $p$-functional on a convex neighborhood]\label{def:frechet_p}
Let $(\m,g)$ be a Riemannian manifold with distance $d$, and let $\mathcal N\subset M$ be geodesically convex.
For a dataset $D=(x_1,\dots,x_n)\in \mathcal N^n$ and $p> 1$, define
\[
F_D(y) \;:=\; \frac{1}{np}\sum_{i=1}^n d(y,x_i)^p, \qquad y\in \mathcal N.
\]
 Define the constrained empirical Fr\'echet $p$-mean $$
\mu_p(D)\;:=\;\displaystyle\arg\min_{y\in \mathcal N}F_D(y),
$$
whenever the minimizer exists and is unique.
\end{definition}
The sensitivity of the Fr\'echet mean (special case of $p=2$) under the DP framework has been studied explicitly in \cite{reimherr2021dpmanifold}. However, for other values of $p$, it is still unknown. In the rest of the paper, we strive to solve this problem and make the following assumption:

\begin{assumption}
    We assume $\m$ has sectional curvature bounded below by $K$ and upper bounded sectional curvature $\kappa$.
\end{assumption}
\begin{remark}
As discussed in Section~\ref{sec:preliminaries}, Ricci curvature can be viewed as
an averaged measure of sectional curvatures through a given direction. In many
geometric settings, this assumption is compatible with diffusion-based
mechanisms, since lower sectional curvature bounds naturally control the behavior of Ricci curvature.
Manifolds satisfying upper sectional curvature bounds can also be regarded as
$\operatorname{CAT}(\kappa)$ metric spaces. The parameter $\kappa$ plays an
important role in determining the convexity properties of the distance
function. When $\kappa>0$, geodesic convexity typically holds only inside
sufficiently small geodesic balls. In contrast, when $\kappa<0$, the distance
function is convex whenever it is smooth. In particular, for Hadamard
manifolds, the cut locus is empty and the distance function is globally convex.
The following result, adapted from \cite{afsari2011lp}, provides a
sufficient condition for the existence and uniqueness of generalized
Fr\'echet means.
\end{remark}
 \begin{proposition}
   
Let $(\m,d)$ be a complete Riemannian manifold whose sectional curvatures are
bounded above by $\kappa$, and let $\inj M$ denote its injectivity radius. Define
\begin{equation}\label{eq:rho-Delta-p}
r_{\kappa,p}\;\triangleq\;
\begin{cases}
\displaystyle \frac12 \min\!\left\{\inj M,\ \frac{\pi}{2\sqrt{\kappa}}\right\},
& \text{if } 1 \le p < 2,\\[2ex]
\displaystyle \frac12 \min\!\left\{\inj M,\ \frac{\pi}{\sqrt{\kappa}}\right\}
\; ,
& \text{if } 2 \le p \le \infty.
\end{cases}
\end{equation}
For $\kappa < 0$, we set $\frac{\pi}{2\sqrt{\kappa}}=\infty$. Let $\nu$ be a probability measure on $\m$ with $\operatorname{supp}(\nu)\subset B(o,r)$ for some $r$, and assume $r<r_{\kappa,p}$. Then, for $1<p\le\infty$, the generalized $p$-Fr\'echet mean (the $L^p$ center of mass) exists, is unique, and belongs to $B(o,r)$.
\end{proposition}
To ensure the existence and uniqueness of the generalized Fr\'echet mean, we
assume that the dataset $D$ is supported inside a geodesic ball
$B(o,r)$ with
$r\le r_{\kappa,p}$
whenever the ambient manifold has positive curvature. For Hadamard manifolds,
no such restriction on $D$ is required.

To analyze the sensitivity of the generalized Fr\'echet mean, our approach
relies on strong geodesic convexity properties of the distance function,
including geodesic $k$-strong convexity and Naor–Silberman/Kuwae (NSK) $p$-uniform convexity, together with
precise estimates of the Riemannian gradient of
$F_D(y).$ Both the exponent $p$ and the geometry of the underlying manifold play a
significant role in determining the strength of these convexity properties and
the behavior of $\nabla F_D(y)$. In particular, Hadamard manifolds are generally
more favorable for establishing strong convexity compared with positively
curved manifolds. Accordingly, we divide our analysis into the following three cases: (i) $p\in(1,2]$ for a general Riemannian manifold $\mathcal M$;
(ii) $p\in(2,+\infty)$ when $\mathcal M$ is a Hadamard manifold;
(iii) $p\in(2,+\infty)$ when $\mathcal M$ has positive sectional curvature. Let us consider the first case and introduce the following \cite{zhang2016first}:
\begin{definition}[$k$-strong convexity]\label{def:strong-convexity}
A function $f:\mathcal{M}\to\mathbb{R}$ is said to be geodesically $k$-strongly
convex if for any $x,z\in\mathcal{M}$,
\begin{equation}\label{eq:strong-convexity-subgradient}
f(z)\;\ge\; f(x)\;+\;\big\langle \nabla f(x),\Log_x(z)\big\rangle_x\;+\;\frac{k}{2}\,d^2(x,z),
\end{equation}
or, equivalently, for any geodesic $\gamma$ such that $\gamma(0)=x$, $\gamma(1)=z$ and
$t\in[0,1]$,
\begin{equation}\label{eq:strong-convexity-geodesic}
f(\gamma(t))\;\le\;(1-t)f(x)\;+\;t f(z)\;-\;\frac{k}{2}\,t(1-t)\,d^2(x,z).
\end{equation}
\end{definition}

When $p\in (1,2]$, geodesic $k$-strong convexity of $d(x,y)^p$ can be satisfied if we restrict the domain of the function. The key reason is that one can obtain a positive lower bound for the second
derivative of the function along any geodesic. More precisely, we have the following result.

\begin{proposition}\label{Prop:mustrong convex}
    Let $x,y\in \m$ lie in $B(o,r)$ and $r\le \frac12\min\{\inj \m,\frac{\pi}{2\sqrt{\kappa}} \}$, the distance function $f(u)=\frac1p d(u,y)^p$ is geodesically $k$-strongly convex in $B(o,r)$ for $p\in (1,2]$, where $k$ is given by 
    \begin{align*}
        k = (2r)^{p-2}\min\{(p-1),\; b_{\kappa}(2r)\},
    \end{align*}
    and $b_c(l)$ is defined by
    \begin{equation}
        b_c(l):=
        \begin{cases}
            \sqrt{c}l\cot(\sqrt{c}l) & \text{if}\ c\ge 0 \\
            1 & \text{if}\ c <0.
        \end{cases}
    \end{equation}
\end{proposition}

Now we are ready to study the sensitivity of the generalized Fr\'echet mean. Let $D=\{x_1,x_2,\cdots,x_n\}$ and $D'=\{x_1,x_2,\cdots,x'_n\}$ be adjacent datasets that differ only in the last entry, i.e., $x_n \ne x'_n$. The sensitivity of the generalized Fr\'echet mean can be bounded as follows.

\begin{theorem}\label{Thm:sensifrechet}
   Assume $p\in(1,2]$. Let the dataset $D$ be supported in the geodesic ball $B(o,r)$ with 
$r\le \frac12\min\{\inj \m,\ \frac{\pi}{2\sqrt{\kappa}} \}$. Then the distance $d(\mu_p(D),\mu_p(D'))$ between the generalized sample Fr\'echet means is bounded by
\[
\frac{\Big(2r\big(2-b_{\kappa}(2r)\big)\Big)^{p-1}}{n(4r)^{p-2}\min\{(p-1),\; b_{\kappa}(2r)\}}.
\]
\end{theorem}

Next we verify that the above bound recovers the result of
\cite{reimherr2021dpmanifold} in the special case $p=2$.
First consider the case where $\mathcal M$ is a Hadamard manifold. We have $\kappa <0$ and $b_\kappa(l)=1$. The upper bound turns out to be $2r/n$ since $p-1 \le 1$.
For $\kappa >0$, let $x:=\sqrt{\kappa}2r$ and the function $b_\kappa(2r)=x\cot(x)$. In our setting, since $r \le \frac{\pi}{4\sqrt{\kappa}}$ and $x\in [0,\frac{\pi}{2}]$. It is easy to see that $\tan x \ge x$ in this domain and we have $x\cot x \le 1$. Thus for $p=2$, we have $\min\{(p-1),\ b_{\kappa}(2r)\}=b_{\kappa}(2r)$ and it illustrates 
\begin{align*}
    d(\mu_2(D),\mu_2(D')) \le \frac{2r(2-b_{\kappa}(2r))}{n b_\kappa(2r)}.
\end{align*}

Now we turn to the case that $p > 2$. In this case, $k$-strong convexity no longer works since there is no lower bound of second derivative of $d(x,y)^p$. In order to analyze the perturbation of the generalized Fr\'echet mean, we also need the following notion of convexity:
\begin{definition}[Naor–Silberman/Kuwae's $p$-uniformly convex]\label{ass:ucp_cat0_explicit}
Let $(X,d)$ be a geodesic metric space and let $p\ge 1$.
Assume that for all $x,y,z\in X$, there exists $k_p$ for every geodesic $\gamma$ from $x$ to $y$, and all $t\in[0,1]$,
\begin{equation}\label{eq:UCp_cat0}
d\!\left(z,\gamma_t\right)^{p}
\ \le\
(1-t)\,d(z,x)^{p}+t\,d(z,y)^{p}
\;-\;\frac{k_p}{2}\,t(1-t)\,d(x,y)^{p}.
\end{equation}
\end{definition}

The inequality \eqref{eq:UCp_cat0} is the NSK notion of $p$-uniform convexity. 
By Remark~2.6 in \cite{kuwae2014jensen}, every CAT(0) space (including Hadamard manifolds) is $p$-uniformly convex with the choice of $k_p$ as follows:
\[
k_p\in (0,c_p] \quad \text{with} \quad c_p :=
\begin{cases}
2(p-1), & p\in (1,2],\\[0.6em]
\displaystyle 8/2^p, & p>2,
\end{cases}
\qquad\text{and we set}\qquad
\lambda_p := \frac{k_p}{p}.
\]
However, NSK $p$-uniform convexity is only known for positively curved manifolds when $p=2$. For $p>2$, there is no general guarantee of such convexity. Thus, we focus on Hadamard manifolds at the first stage.

\begin{theorem}[Sensitivity on a Hadamard manifold]\label{thm:sens_hadamard_final}
Let $(\m,g)$ be a Hadamard manifold (complete, simply connected, $\Sec\le 0$).
Fix $r>0$ and set the convex neighborhood $\mathcal N=\overline{B(o,r)}$.
Thus, the NSK $p$-uniform convexity condition holds on $\mathcal N$ with constants $k_p,\lambda_p$ as defined above.
Let $D,D'\in \mathcal N^n$ be adjacent. Then we have 
\[
d(\mu_p(D),\mu_p(D'))
\ \le\
\left(
\frac{(p-1)(4r)^{p-2}}{n\,\lambda_p}\,2r
\right)^{\!\frac{1}{p-1}}.
\]
In particular, for $p=2$ (which means $\lambda_2=1$), this gives $d(\mu_2(D),\mu_2(D'))\le 2r/n$.
\end{theorem}

%===============================================================================
% Part II. Compact manifold: remove separate Log-Lipschitz assumption;
% compute L_log from curvature bounds (Assumption 1) inside the theorem proof.
%===============================================================================

Thanks to the NSK $p$-uniform convexity property, we can recover the same upper bound on the sensitivity of the Fr\'echet mean by setting $p=2$ in the case where $\mathcal M$ is a Hadamard manifold.
For positively curved manifolds, results from \cite{kuwae2014jensen} and \cite{ohta2007convexities} show that if a metric space is $\operatorname{CAT}(\kappa)$ with $\kappa>0$, then any geodesic ball $N$ with radius
$r \le \frac{\pi}{4\sqrt{\kappa}}$
is an NSK $2$-uniformly convex space with parameter
$k_2
=
4r\sqrt{\kappa}\cot(\sqrt{\kappa}\,2r)=2b_\kappa(2r).$
For other values of $p$, the sufficient condition on the radius for NSK $p$-uniform convexity is still unknown. Our strategy is to derive a slightly weaker $p$-convexity property ($p> 2$) from the NSK $2$-uniform convexity condition, which we refer to as \emph{$p$-median convexity}. Although weaker, this property is sufficient to establish a slightly looser upper bound on the sensitivity of the generalized Fr\'echet mean.

\begin{theorem}[Compact-manifold case]\label{thm:sens_compact_final}
Let $\m$ be a manifold with positive curvature bounded by $\kappa$. Let $\mathcal{N}$ be a geodesic ball with radius $r$ such that it is NSK uniformly $2$-convex with parameter $k_2$. 
For $p\ge 2$ and adjacent $D,D'\in \mathcal N^n$,
\[
 d(\mu_p(D),\mu_p(D'))\le \left(\frac{2^{p-1}p(p-1)(4r)^{p-2}(2-b_\kappa(2r))2r}{nb_{\kappa}(2r) }\right)^{\frac{1}{p-1}}.
\]
\end{theorem}
\begin{remark}
    One observes that the sensitivity of the generalized Fr\'echet mean converges to $0$ as the sample size increases to infinity. Interestingly, the exponent $p$ appears to influence the rate of convergence. Theorems~\ref{thm:sens_hadamard_final} and~\ref{thm:sens_compact_final} suggest that this rate may decrease as $p$ increases. This phenomenon becomes particularly intriguing when studying central limit theorems under differential privacy guarantees. We will present a brief discussion on this in Section \ref{sec:discussion}.

\end{remark}

\section{Algorithmic implementation of the private generalized Fr\'echet mean}
\label{sec:gfm_algorithm}

% Add the following lines to the preamble if they are not already present:
% \usepackage{algorithm}
% \usepackage{algpseudocode}
% \algrenewcommand\algorithmicrequire{\textbf{Input:}}
% \algrenewcommand\algorithmicensure{\textbf{Output:}}

% The exact RDP guarantees are stated for the exact diffusion laws
% B_t(a) and X_t^o(a). The two samplers below are intrinsic time-discretizations
% used for implementation and numerical experiments.

Define the global sensitivity of the generalized Fr\'echet mean by
$$\displaystyle \Delta_{\mu_p}:=\sup_{D\sim D'} d\!\bigl(\mu_p(D),\mu_p(D')\bigr),$$
where $D\sim D'$ means that $D$ and $D'$ differ in at most one entry.  In practice, one
substitutes any valid upper bound for $\Delta_{\mu_p}$.  The results established above give
three convenient choices.
In what follows we write $\Delta:=\Delta_{\mu_p}$ for brevity.

A practical advantage of both mechanisms is that they are
\emph{normalization-free}.  One does not need to evaluate the heat kernel $p(x,z,t)$, the
stationary density of the Langevin diffusion, or any partition function.  Instead, one samples
by simulating an intrinsic SDE on $\m$.  This is especially appealing on manifolds where
closed-form densities are unavailable but geodesics, exponential maps, or logarithm maps are
computationally accessible.  On Hadamard manifolds, the emptiness of cut loci further simplifies
implementation because $\Exp_x$ and $\Log_x$ are globally defined.

For concreteness we describe simple intrinsic time-discretizations in Algorithm~\ref{alg:bm_sampler} and \ref{alg:langevin_sampler} for Brownian motion and Langevin dynamics, respectively. They are sufficient for
implementation and numerical experiments, and they preserve the manifold constraint at every
step because each update is performed in the tangent space and mapped back by $\Exp$.
The exact RDP guarantees in Theorems~\ref{Thm:BMM} and~\ref{Thm:RDP2} are stated for the
exact diffusion laws; in practice, one uses a sufficiently fine discretization, or exact
sampling whenever a closed-form diffusion sampler is available.\footnote{The exact sampling for Brownian motion is only available on the sphere and only works well for time horizons $t > 0.05$. Refer to \cite{mijatovic2020note} for details.} Finally, we summarize the $(\alpha,\varepsilon)$-RDP release procedure of the generalized Fr\'echet mean in Algorithm \ref{alg:private_gfm} with theoretical guarantee in Proposition \ref{prop:rdp_gfm_algorithm}.

\begin{algorithm}[h]
\caption{Intrinsic sampler for the BM mechanism $B_t(a)$}
\label{alg:bm_sampler}
{ \normalsize
\begin{algorithmic}[1]
\Require Starting point $a\in \m$, diffusion time $t>0$, number of steps $N_{\mathrm{step}}\ge 1$
\Ensure An intrinsic approximation of $B_t(a)$
\State Set $h\gets t/N_{\mathrm{step}}$ and $Y_0\gets a$
\For{$k=0,1,\ldots,N_{\mathrm{step}}-1$}
    \State Choose an orthonormal frame $u_k:\mathbb{R}^m\to T_{Y_k}\m$
    \State Sample $Z_k\sim N(0,I_m)$ in $\mathbb{R}^m$
    \State Set $\xi_k\gets u_k Z_k\in T_{Y_k}\m$
    \State Update $Y_{k+1}\gets \Exp_{Y_k}\bigl(\sqrt{2h}\,\xi_k\bigr)$
\EndFor
\State \Return $Y_{N_{\mathrm{step}}}$
\Statex
\Statex \emph{Implementation note.} Each step only requires sampling a standard Gaussian in $\mathbb{R}^m$ and evaluating one exponential map. No heat kernel evaluation or normalization constant is needed.
\end{algorithmic}}
\end{algorithm}

\begin{algorithm}[h]
\caption{Intrinsic sampler for the Langevin mechanism $X_t^o(a)$ on a Hadamard manifold}
\label{alg:langevin_sampler}
{ \normalsize
\begin{algorithmic}[1]
\Require Starting point $a\in \m$, anchor $o\in \m$, drift parameter $\lambda>0$, diffusion time $t>0$, number of steps $N_{\mathrm{step}}\ge 1$
\Ensure An intrinsic approximation of $X_t^o(a)$ generated by $L=\Delta-\langle \nabla V,\nabla\rangle$ with $V(x)=\frac{\lambda}{2}d^2(o,x)$
\State Set $h\gets t/N_{\mathrm{step}}$ and $Y_0\gets a$
\For{$k=0,1,\ldots,N_{\mathrm{step}}-1$}
    \State Choose an orthonormal frame $u_k:\mathbb{R}^m\to T_{Y_k}\m$
    \State Sample $Z_k\sim N(0,I_m)$ in $\mathbb{R}^m$
    \State Set $\xi_k\gets u_k Z_k\in T_{Y_k}\m$
    \State Compute the drift vector $b_k\gets \lambda\,\Log_{Y_k}(o)\in T_{Y_k}\m$
    \State Update $Y_{k+1}\gets \Exp_{Y_k}\bigl(h\,b_k+\sqrt{2h}\,\xi_k\bigr)$
\EndFor
\State \Return $Y_{N_{\mathrm{step}}}$
\Statex
\Statex \emph{Implementation note.} On a Hadamard manifold, $\Log_{Y_k}(o)$ is globally well-defined and each step only requires Gaussian sampling, one logarithm map to the public anchor $o$, and one exponential map. No stationary density or normalization constant is evaluated.
\end{algorithmic}}
\end{algorithm}

\begin{proposition}[RDP guarantee for the private generalized Fr\'echet mean]
\label{prop:rdp_gfm_algorithm}
Assume that $\widehat\mu_p(D)=\mu_p(D)$ exists uniquely and that $\Delta$ is any valid upper
bound on its global sensitivity.
\begin{enumerate}[label=(\roman*),leftmargin=*]
    \item In the BM regime of Theorem~\ref{Thm:BMM}, the release
    $\widetilde\mu_p(D)=B_t(\widehat\mu_p(D))$ with $t$ chosen as in Algorithm~\ref{alg:private_gfm} is
    $(\alpha,\varepsilon)$-RDP.
    \item In the Hadamard/Langevin regime of Theorem~\ref{Thm:RDP2}, the release
    $\widetilde\mu_p(D)=X_t^o(\widehat\mu_p(D))$ with $t$ chosen as in Algorithm~\ref{alg:private_gfm} is
    $(\alpha,\varepsilon)$-RDP.
\end{enumerate}
\end{proposition}

\begin{proof}
Both claims are immediate from the diffusion-based RDP theorems once the deterministic
summary $f(D)$ is specialized to $\widehat\mu_p(D)=\mu_p(D)$ and its sensitivity bound $\Delta$
is substituted into the formulas of Theorems~\ref{Thm:BMM} and~\ref{Thm:RDP2}.
\end{proof}

\begin{algorithm}[h]
\caption{$(\alpha,\varepsilon)$-RDP release of the generalized Fr\'echet mean}
\label{alg:private_gfm}
{ \normalsize
\begin{algorithmic}[1]
\Require Dataset $D=(x_1,\ldots,x_n)$, order $\alpha>1$, privacy budget $\varepsilon>0$, exponent $p>1$, and a choice of diffusion mechanism
\Ensure A private release $\widetilde\mu_p(D)$ of $\mu_p(D)$ satisfying $(\alpha,\varepsilon)$-RDP
\State Compute the non-private generalized Fr\'echet mean $\widehat\mu_p(D)\gets \mu_p(D)$
\State Compute a valid upper bound $\Delta$ on the global sensitivity of $\mu_p$,
\Statex \hspace{1.5em}applying Theorem~\ref{Thm:sensifrechet} when $1<p\le 2$;
\Statex \hspace{1.5em}applying Theorem~\ref{thm:sens_hadamard_final} when $p>2$ and $\m$ is Hadamard;
\Statex \hspace{1.5em}applying Theorem~\ref{thm:sens_compact_final} in the compact positive-curvature regime.
\If{the BM mechanism is used}
    \State Calibrate the diffusion time by Theorem~\ref{Thm:BMM}: $t\gets \frac{\log(2\varepsilon)-\log\bigl(2\varepsilon-K\alpha\Delta^2\bigr)}{2K}$
    \Statex \hspace{1.5em}and in the Euclidean case $K=0$, use $t\gets \alpha\Delta^2/(4\varepsilon)$
    \State Sample $\widetilde\mu_p(D)$ from $B_t\bigl(\widehat\mu_p(D)\bigr)$ using Algorithm~\ref{alg:bm_sampler}
\Else
    \State Choose an anchor $o\in \m$ independently of the confidential data and select $\lambda>K$
    \State Calibrate the diffusion time by Theorem~\ref{Thm:RDP2}: $t\gets \frac{1}{2(\lambda-K)}\log\!\left(1+\frac{(\lambda-K)\alpha\Delta^2}{2\varepsilon}\right)$
    \State Sample $\widetilde\mu_p(D)$ from $X_t^o\bigl(\widehat\mu_p(D)\bigr)$ using Algorithm~\ref{alg:langevin_sampler}
\EndIf
\State \Return $\widetilde\mu_p(D)$
\Statex
\Statex \emph{Privacy guarantee.} By Theorems~\ref{Thm:BMM} and~\ref{Thm:RDP2}, the returned value satisfies $(\alpha,\varepsilon)$-RDP once the chosen sensitivity bound $\Delta$ is valid.
\end{algorithmic}}
\end{algorithm}
\clearpage
\section{Numerical Experiments}
\label{sec:experiment}

To illustrate the utility of our proposed mechanisms, we conduct numerical simulations to compare them with existing geometry-aware privacy mechanisms, namely the Riemannian Laplace (RL) mechanism \cite{reimherr2021dpmanifold} and the Exponential-Wrapped Gaussian (EWG) mechanism for RDP \cite{jiang2026expwrap}, on both a compact manifold—the sphere (Section~\ref{sec:sphere})—and a Hadamard manifold—the hyperbolic space (Section~\ref{sec:hyperbolic}). On the sphere, we evaluate our BM mechanism (Theorem~\ref{Thm:BMM}) against the RL mechanism; on the hyperbolic space, we evaluate our Langevin mechanism (Theorem~\ref{Thm:RDP2}) against both the RL mechanism and the EWG mechanism. 

In both settings, we generate a dataset $D=(x_1,\dots,x_n)$ of size $n$ from a geodesic ball of radius $r$, and the goal is to release an $(\alpha,\varepsilon)$-RDP-compliant sample Fr\'echet mean based on $D$. To assess utility, we compute the distance $d(\hat{\mu},\tilde{\mu})$ between the confidential sample Fr\'echet mean $\hat{\mu}$ and the privatized output $\tilde{\mu}$. Naturally, a smaller distance indicates better utility. We compare these distances across a range of privacy budgets $(\alpha,\varepsilon)$ with $\alpha=2$ and $\varepsilon\in\{0.1,0.2,0.3,0.4,0.5,0.7,1,1.5,2,2.5,3\}$. For each privacy level, we repeat the simulation 1000 times and report the average distances, where red circles represent our mechanisms and blue triangles represent the competing mechanisms. The shaded regions indicate standard errors around the average distances; see Figures~\ref{fig_sphere_RL}-\ref{fig_hyperbolic_EWG}.

We briefly recall the RL and EWG mechanisms. Given a private summary $f(D)$, the RL mechanism achieves $\varepsilon$-DP by outputting a sample from a Riemannian Laplace distribution with footpoint $f(D)$ and rate $\sigma=\Delta/\varepsilon$, which has density $p(x)\propto \exp\!\big(-d(f(D),x)/(\Delta/\varepsilon)\big)$.\footnote{The rate $\sigma=\Delta/\varepsilon$ is achieved only on homogeneous manifolds. For non-homogeneous manifolds, we have $\sigma=2\Delta/\varepsilon$.} Note that any $\varepsilon$-DP mechanism also achieves $(\alpha,\varepsilon)$-RDP, and furthermore $\varepsilon$-DP is equivalent to $(\alpha,\varepsilon)$-RDP as $\alpha\to\infty$. Here we use a sharper conversion (see the Supplementary Materials): an $\varepsilon^*$-DP mechanism achieves $(\alpha,\varepsilon(\alpha))$-RDP with
\begin{equation}\label{eq:dp2rdp}
    \varepsilon(\alpha) = \frac{1}{\alpha-1}
    \log\!\left(
    \frac{1}{\exp(\varepsilon^*)+1}\exp(\alpha\varepsilon^*)
    +
    \frac{\exp(\varepsilon^*)}{\exp(\varepsilon^*)+1}\exp(-\alpha\varepsilon^*)
    \right).
\end{equation}
Similarly, the EWG mechanism achieves $(\alpha,\varepsilon)$-RDP by sampling from an EWG distribution with footpoint $p_0$, center $f(D)$, and rate $\sigma=\Delta/\sqrt{2\varepsilon/\alpha}$ (see \cite{jiang2026expwrap} for the definition of the EWG distribution). Note that the choice of footpoint $p_0$ plays a role analogous to the anchor point $o$ in our Langevin mechanism, since both must be chosen independently of the confidential dataset. In our implementation, we use the same point as the footpoint for the EWG mechanism and as the anchor for the Langevin mechanism.

\subsection{Sphere}
\label{sec:sphere}

For simulations on the sphere $S^m$, we focus on $m=2$ and sample a dataset $D$ uniformly from the geodesic ball $B(p,r)$ centered at $p=(0,0,1)$ with radius $r=\pi/5$. We implement two mechanisms that output an $(\alpha,\varepsilon)$-RDP-compliant sample Fr\'echet mean: our BM mechanism and the RL mechanism. For each privacy budget $\varepsilon$, we compute the corresponding $\varepsilon^*$ for $\varepsilon^*$-DP using \eqref{eq:dp2rdp} and then implement the RL mechanism with $\sigma=\Delta/\varepsilon^*$. The comparison between the BM and RL mechanisms is displayed in Figure~\ref{fig_sphere_RL}.

\begin{figure}[h!]
    \centering
    \stackunder[5pt]{
    \includegraphics[width=0.30\textwidth]{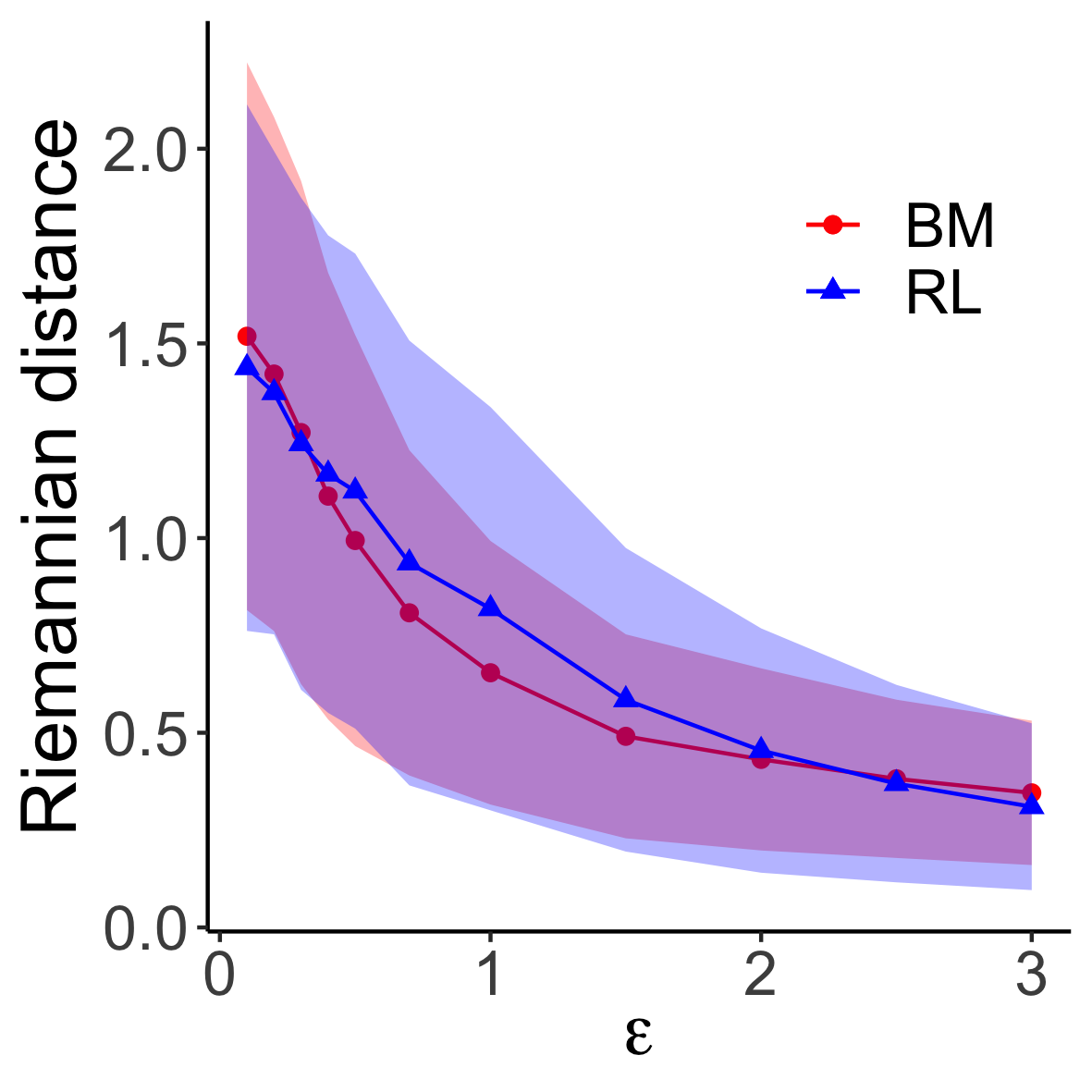}}{$n =10$}
    \hfill
    \stackunder[5pt]{
    \includegraphics[width=0.30\textwidth]{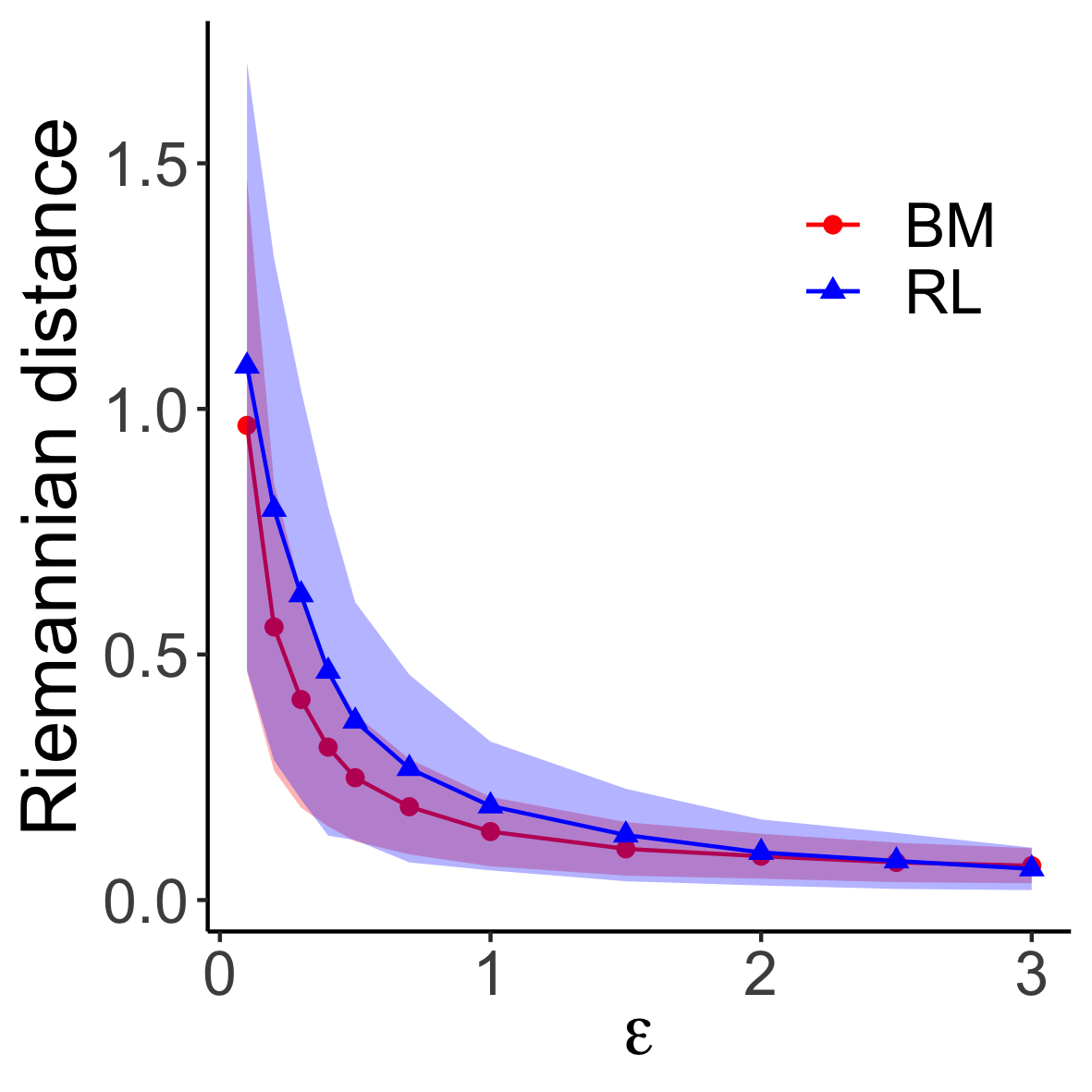}}{$n=50$}
    \hfill
    \stackunder[5pt]{
    \includegraphics[width=0.30\textwidth]{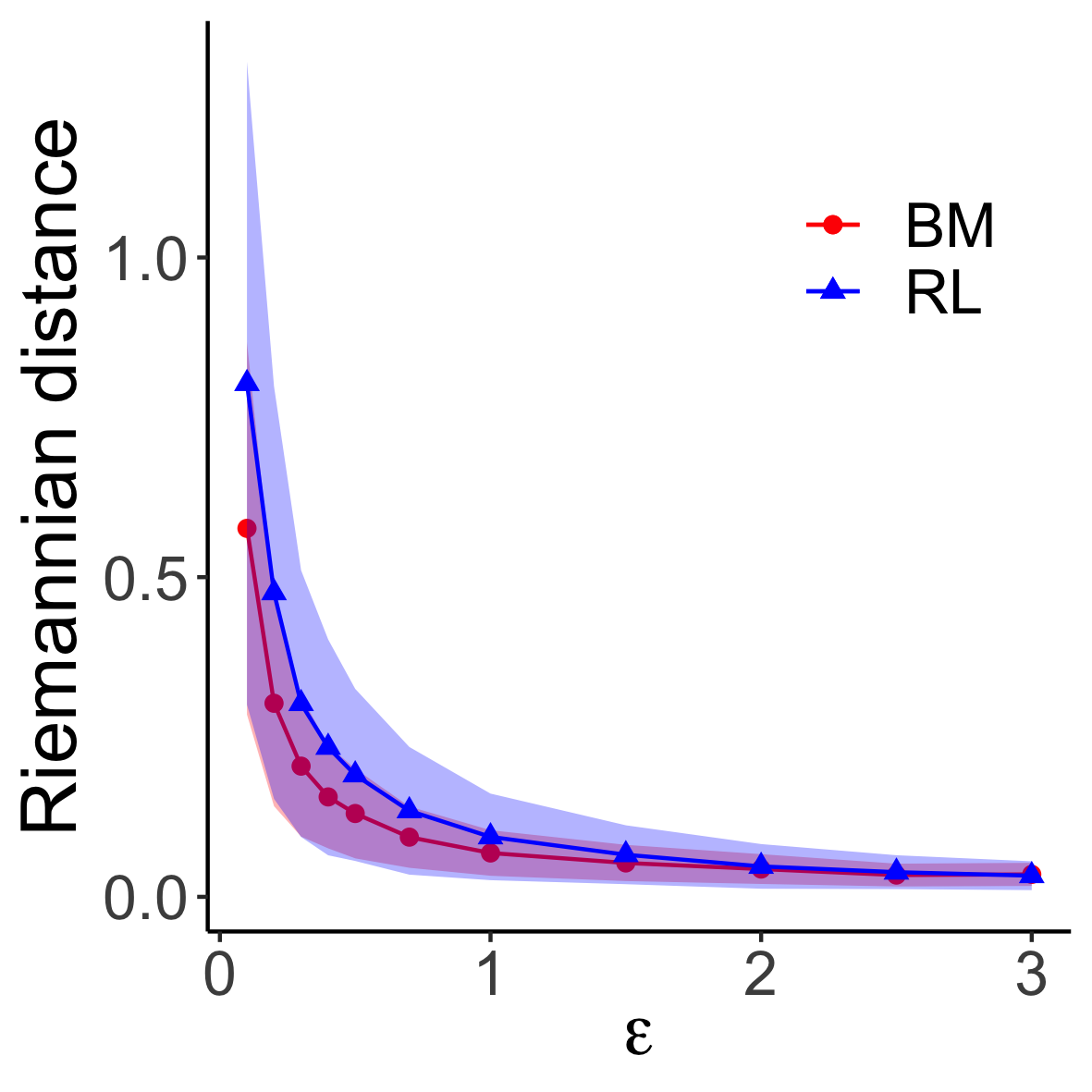}}{$n=100$}
    \caption{Utility comparison of BM and RL mechanisms across different sample size $n \in \{10, 50, 100\}$. Blue lines with triangular symbols show the distances of the RL mechanism, and red lines with circular symbols represent the Riemannian distances of our BM mechanism.}
    \label{fig_sphere_RL}
\end{figure}

Figure~\ref{fig_sphere_RL} illustrates the expected privacy-utility trade-off on the sphere. For both mechanisms, the distance to the true sample Fr\'echet mean decreases as $\varepsilon$ increases, since a larger privacy budget corresponds to less injected noise. The distance also decreases with the sample size $n$, reflecting the decay of the global sensitivity of the sample Fr\'echet mean as $n$ grows. Overall, the BM mechanism remains highly competitive and outperforms the RL mechanism across a wide range of privacy budgets. This behavior is expected: for our BM mechanism, the heat kernel aligns naturally with the definition of RDP, whereas the RL mechanism requires a conversion between different privacy notions.

\subsection{Hyperbolic Space}
\label{sec:hyperbolic}

We use the hyperboloid model $\mathbb{H}_m$ to represent the hyperbolic space. For background material such as the Riemannian distance and expressions for the exponential map under the hyperboloid model, we refer to \cite{nickel2017poincare,nagano2019wrapped}. We fix $m=2$ and sample a dataset $D$ uniformly from the geodesic ball $B(p,r)$ centered at $p=(1,0,0)$ with radius $r=3$. 

We implement three mechanisms that output an $(\alpha,\varepsilon)$-RDP-compliant sample Fr\'echet mean: our Langevin mechanism, the RL mechanism, and the EWG mechanism. To implement the Langevin mechanism, we must specify two parameters: the anchor $o$ and $\lambda$. We fix $\lambda=1.1$, where the lower bound for $\lambda$ is $1$ on $\mathbb{H}_2$, and we consider two scenarios for selecting the anchor. In Scenario~1, representing the case where there is publicly available information about the confidential data, we set the anchor to be the center $p$ of the geodesic ball $B(p,r)$. In Scenario~2, representing the lack of public information, we randomly select a point within $B(p,r)$ as the anchor. 

For the implementation of the EWG mechanism, we only need to choose the footpoint, which, as mentioned above, is set to be the same as the anchor in the Langevin mechanism. Finally, for the RL mechanism, we compute the corresponding DP budget $\varepsilon^*$ using \eqref{eq:dp2rdp} and then apply the standard implementation with rate $\sigma=\Delta/\varepsilon^*$. The comparison between the Langevin and RL mechanisms is shown in Figure~\ref{fig_hyperbolic_RL}, and the comparison between the Langevin and EWG mechanisms is shown in Figure~\ref{fig_hyperbolic_EWG}.

\begin{figure}[h!]
    \centering
    \stackunder[5pt]{
    \includegraphics[width=0.30\textwidth]{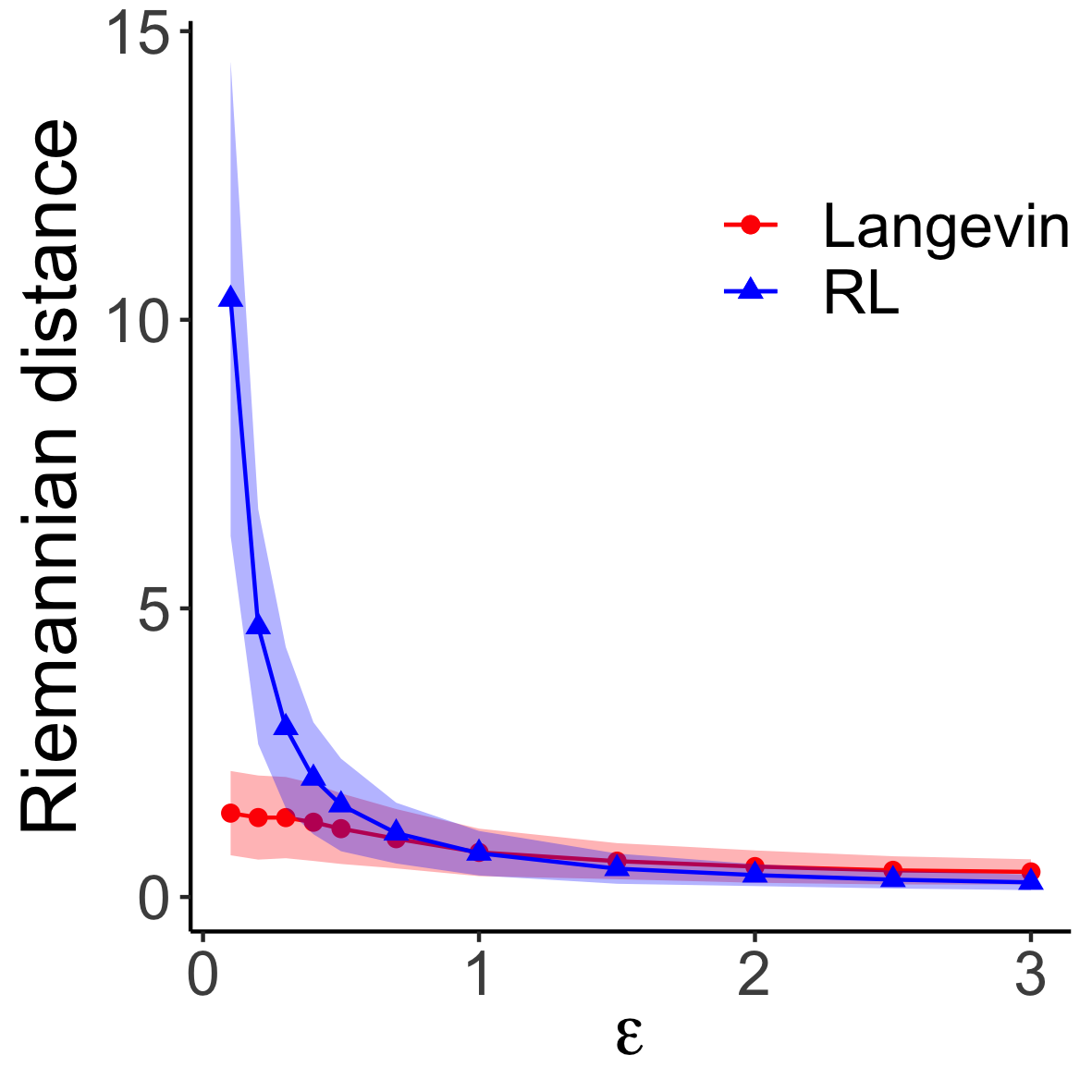}}{}
    \hfill
    \stackunder[5pt]{
    \includegraphics[width=0.30\textwidth]{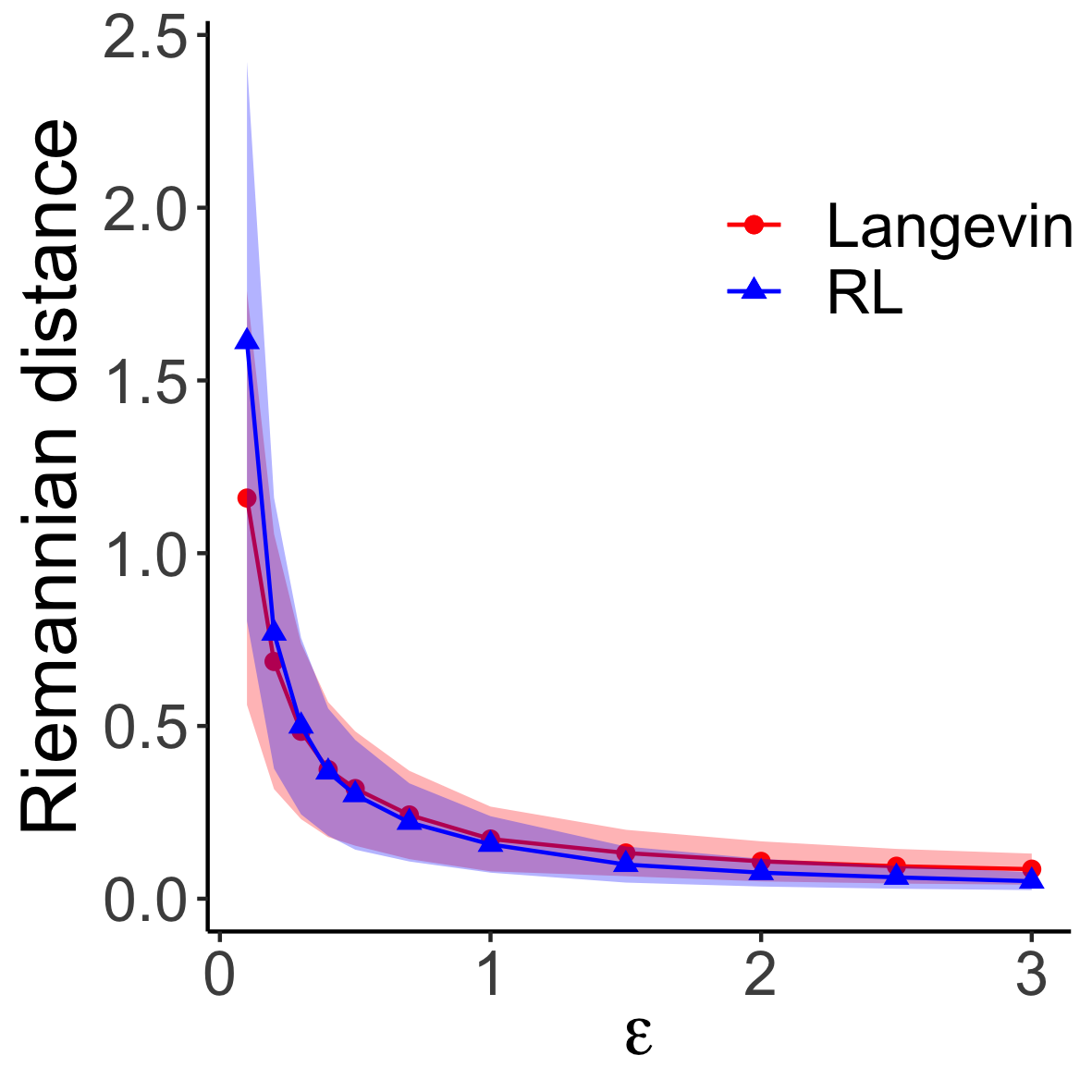}}{}
    \hfill
    \stackunder[5pt]{
    \includegraphics[width=0.30\textwidth]{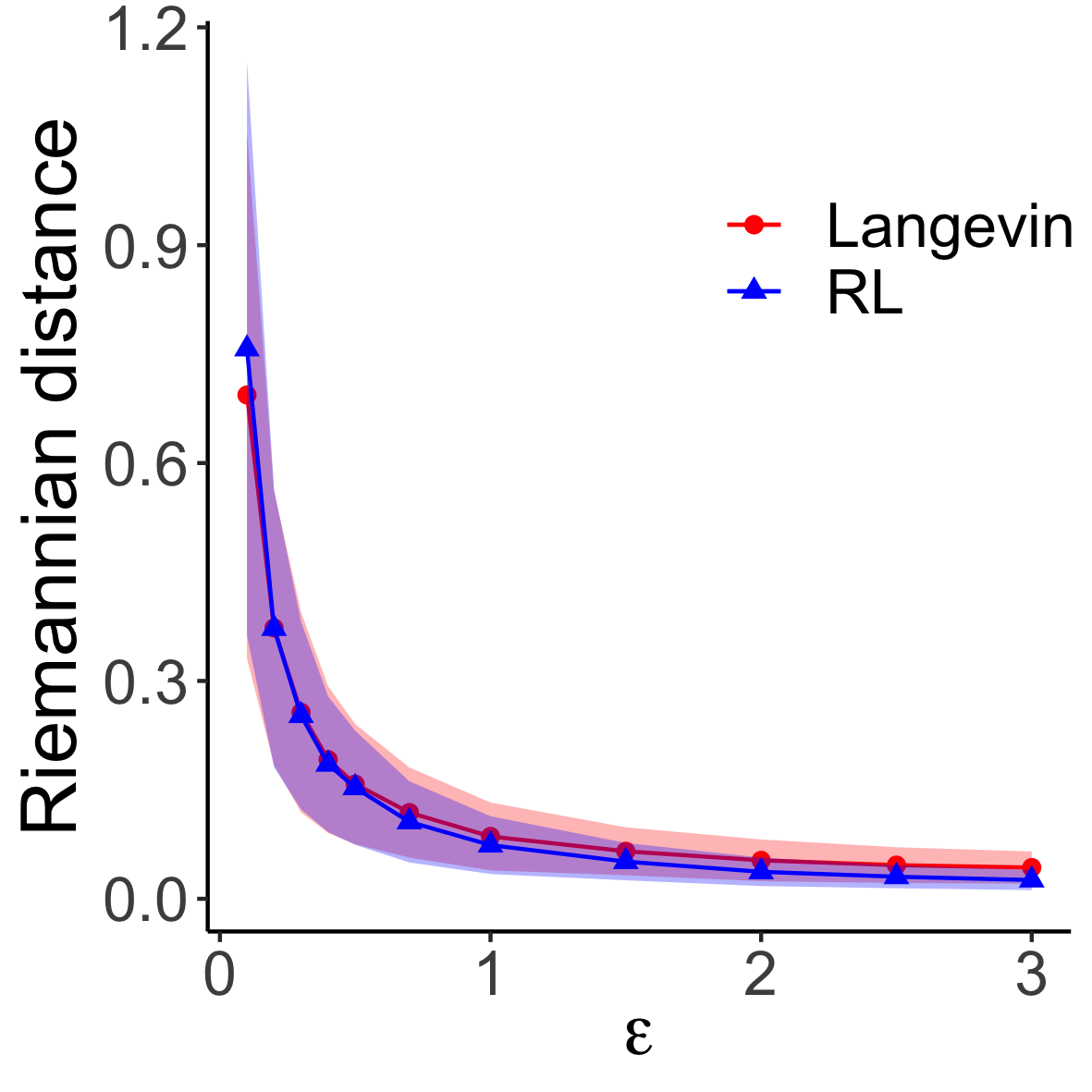}}{}
    \stackunder[5pt]{
    \includegraphics[width=0.30\textwidth]{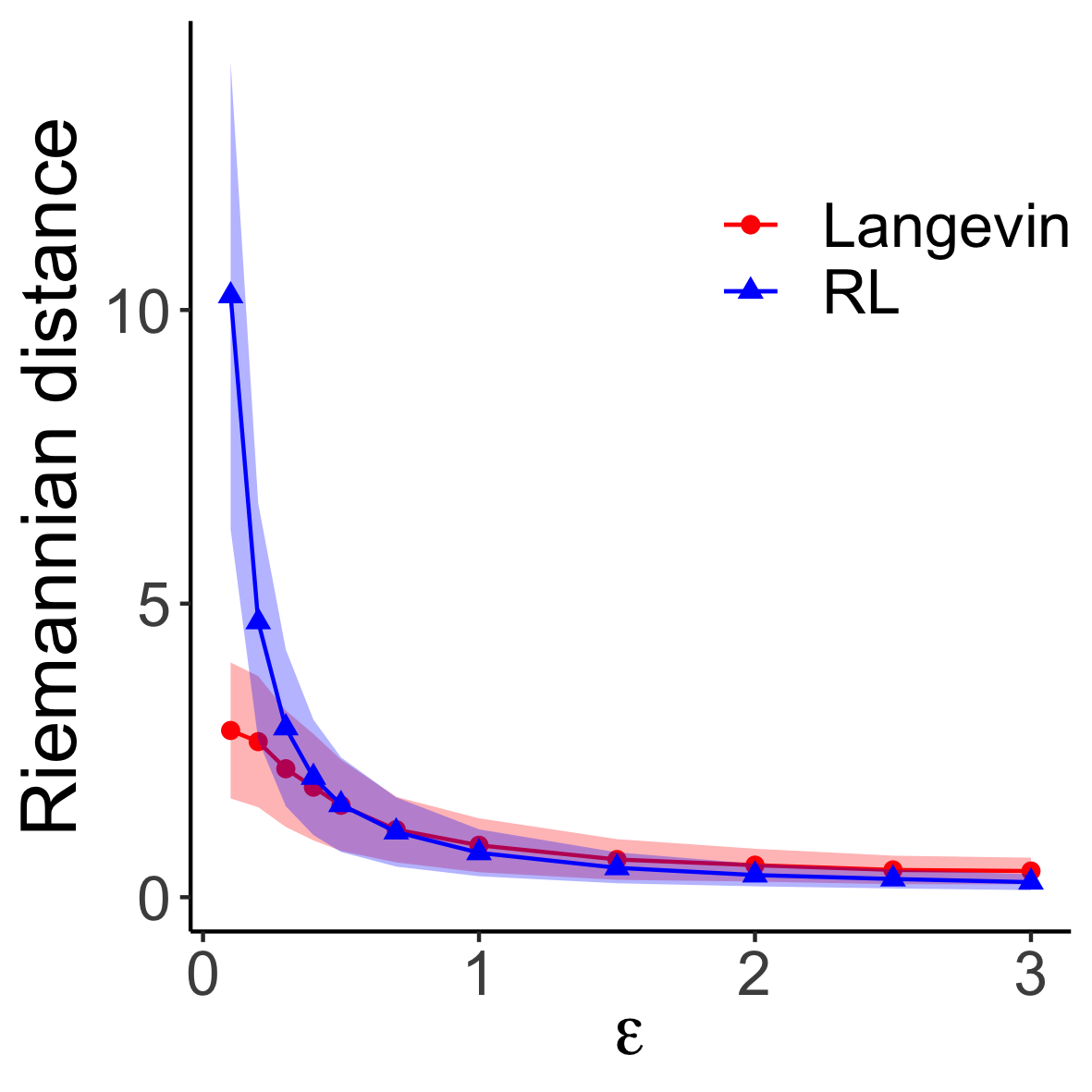}}{$n =10$}
    \hfill
    \stackunder[5pt]{
    \includegraphics[width=0.30\textwidth]{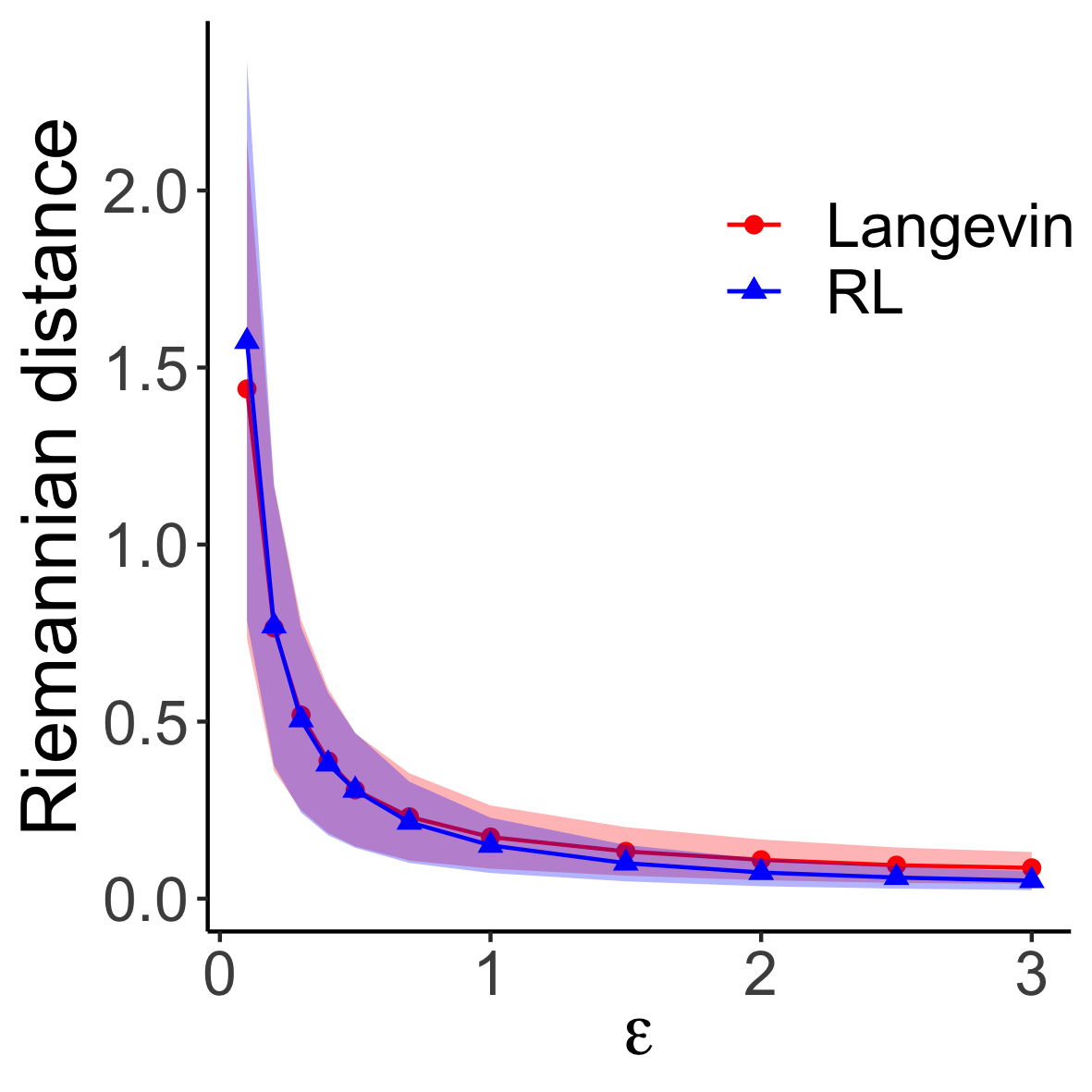}}{$n=50$}
    \hfill
    \stackunder[5pt]{
    \includegraphics[width=0.30\textwidth]{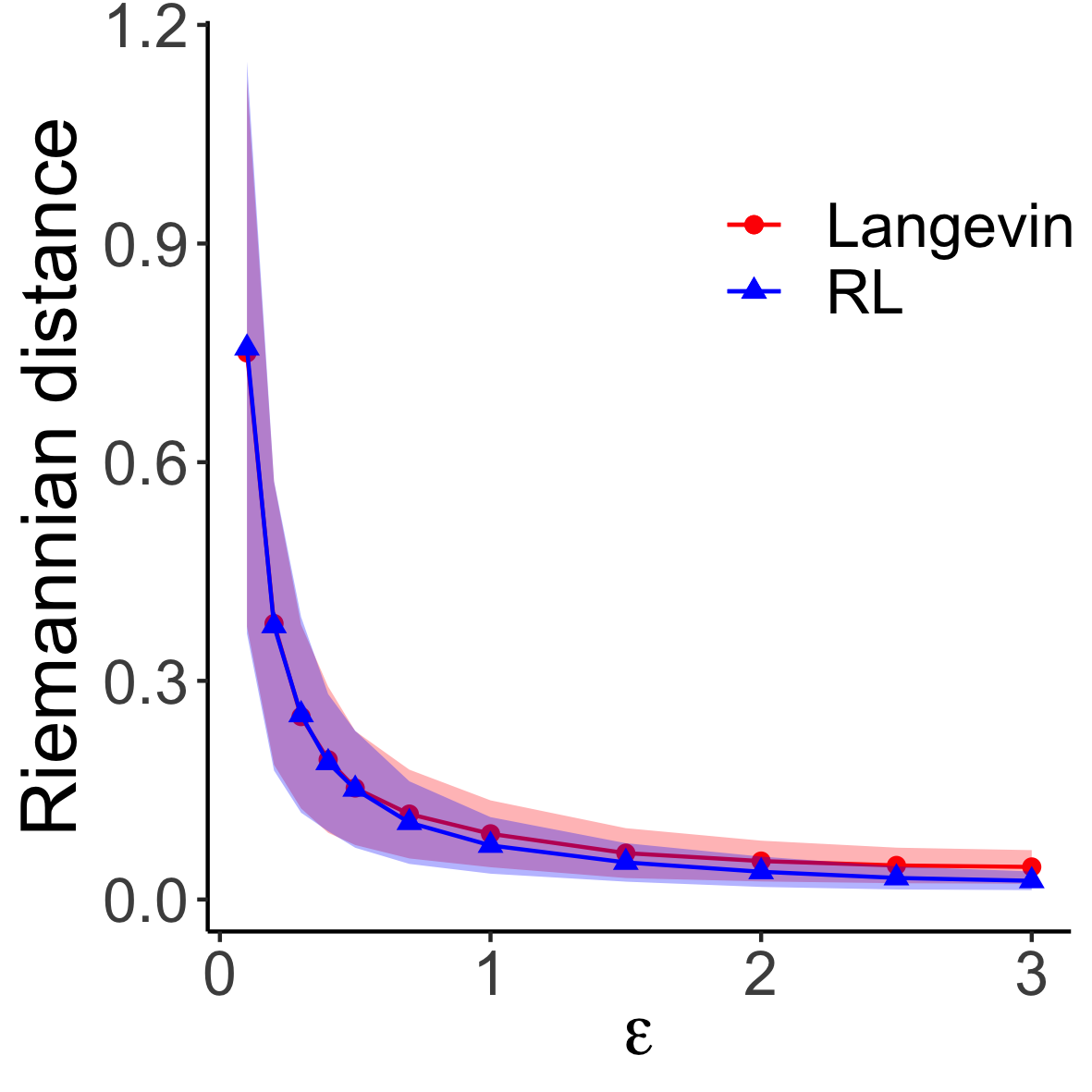}}{$n=100$}
    \caption{Utility comparison of Langevin and RL mechanisms across different sample size $n \in \{10, 50, 100\}$. Blue lines with triangular symbols show the distances of the RL mechanism, and red lines with circular symbols represent the Riemannian distances of our Langevin mechanism. The results are shown for Scenario 1 (top) and Scenario 2 (bottom).}
    \label{fig_hyperbolic_RL}
\end{figure}

\begin{figure}[h!]
    \centering
    \stackunder[5pt]{    \includegraphics[width=0.30\textwidth]{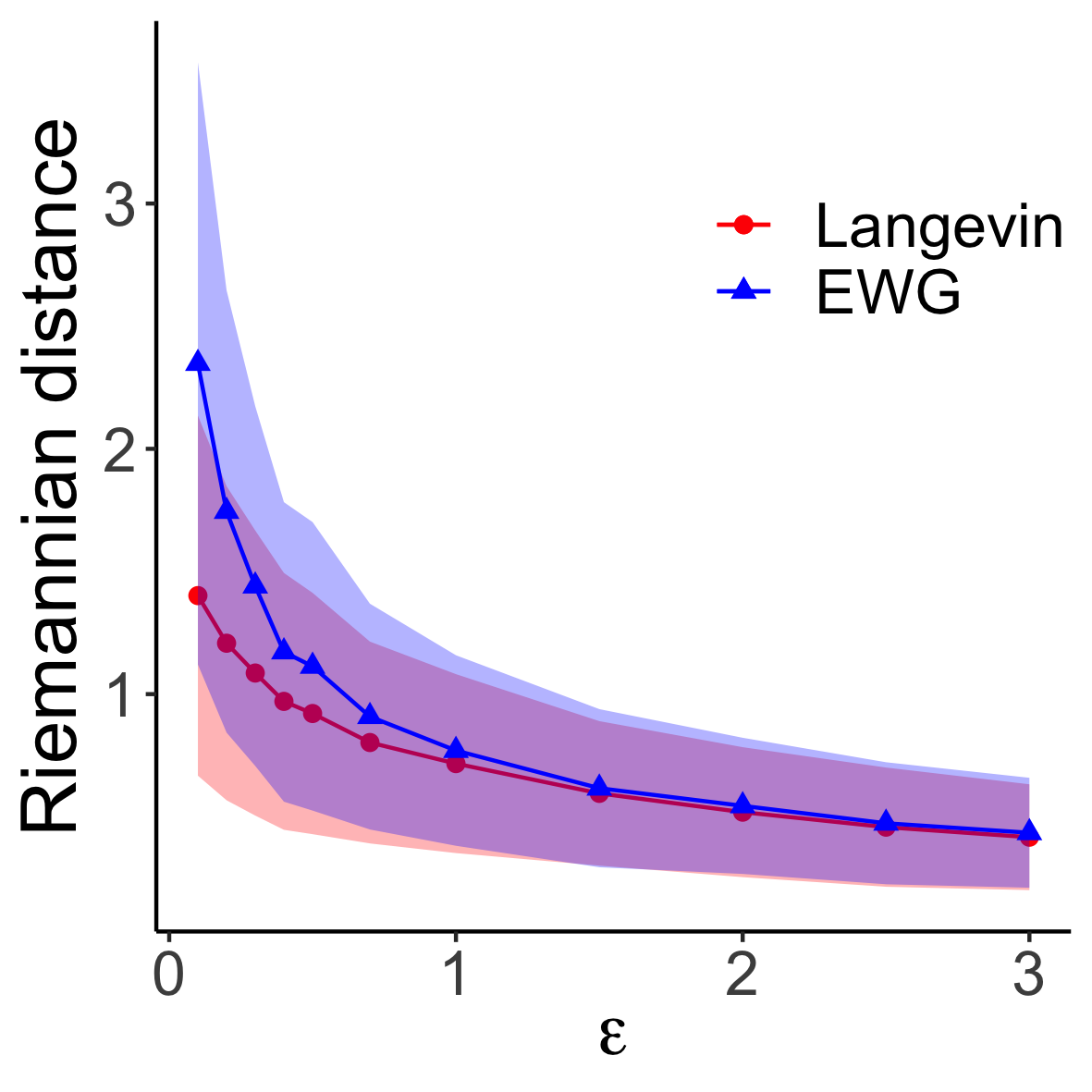}}{}
    \hfill
    \stackunder[5pt]{
    \includegraphics[width=0.30\textwidth]{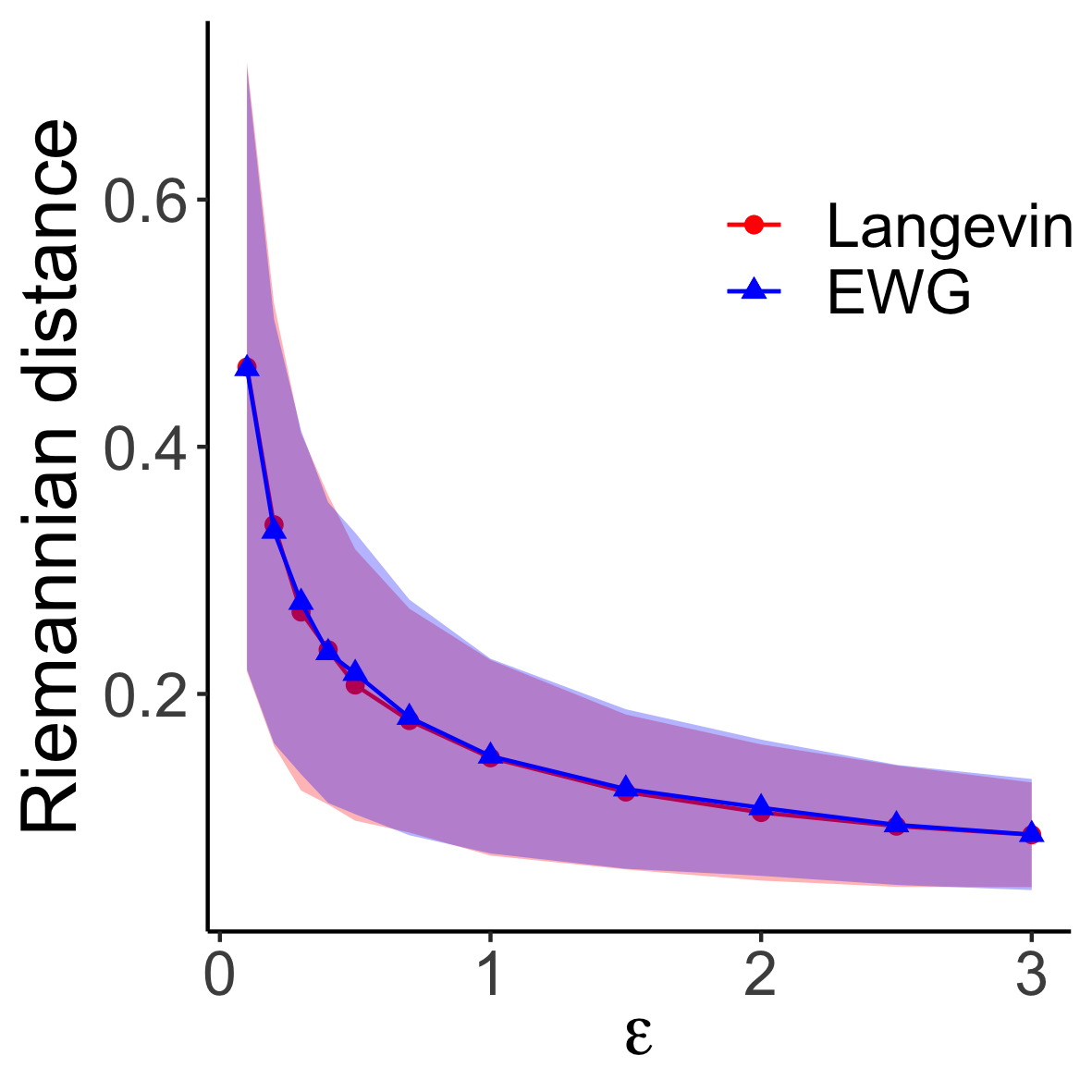}}{}
    \hfill
    \stackunder[5pt]{
    \includegraphics[width=0.30\textwidth]{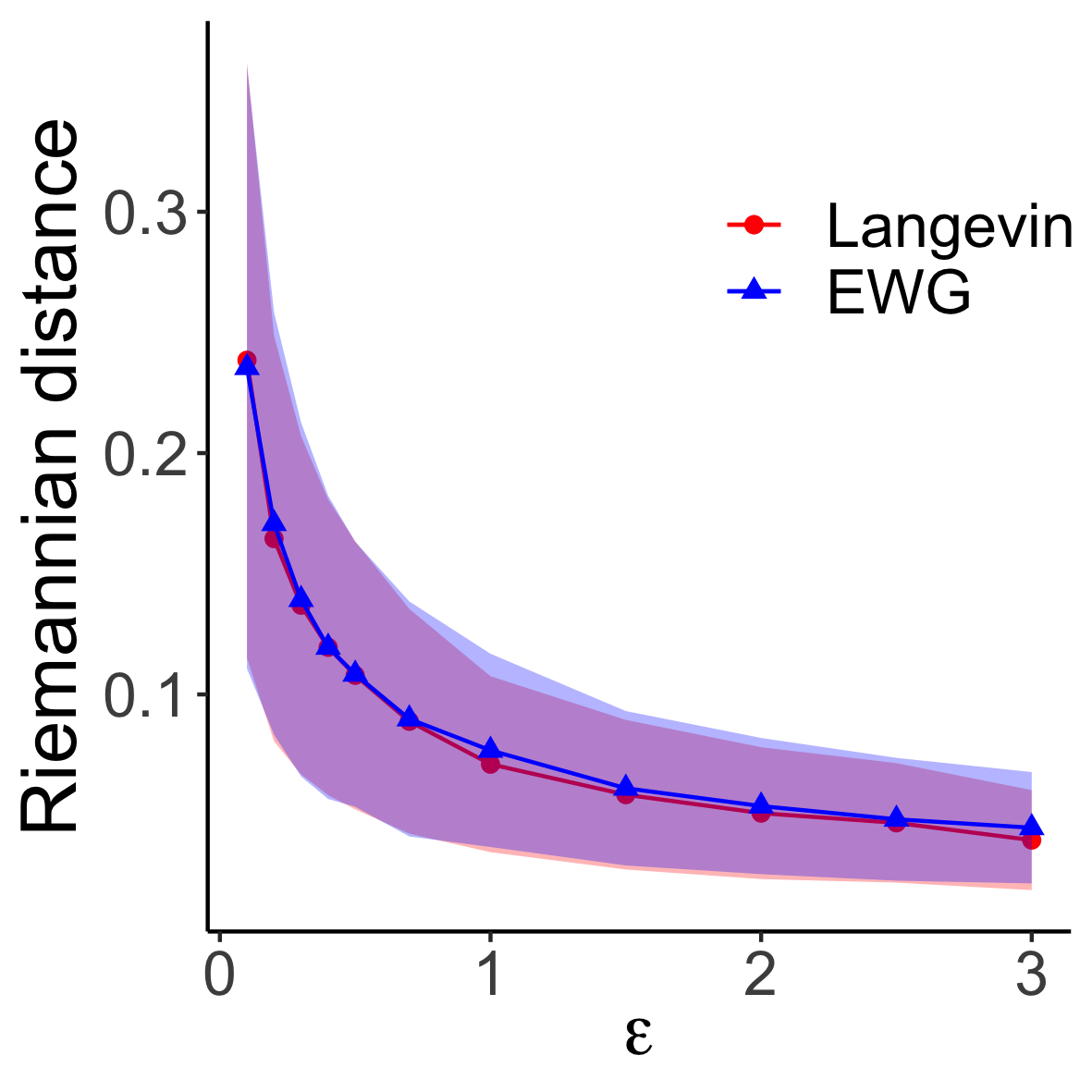}}{}
    \stackunder[5pt]{
    \includegraphics[width=0.30\textwidth]{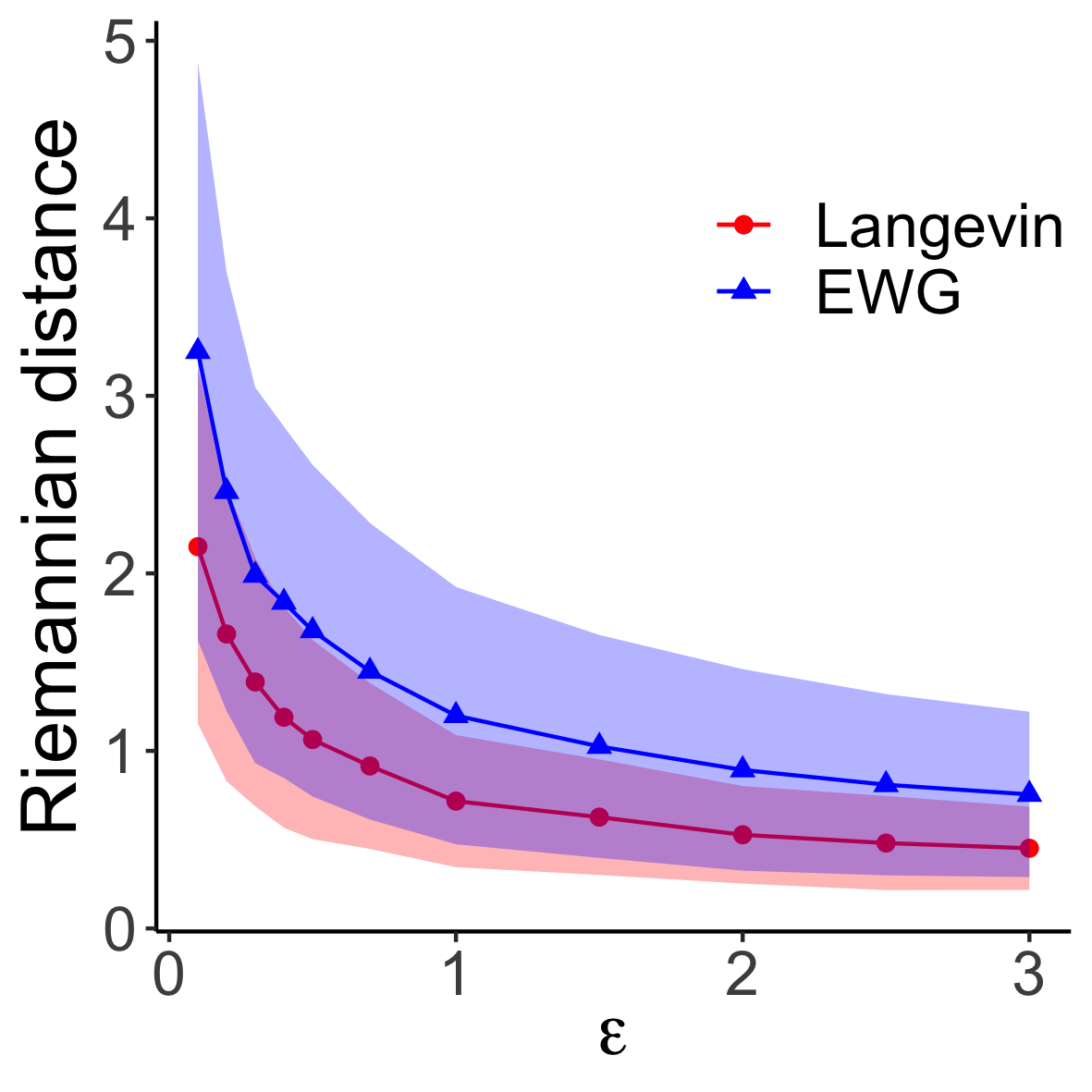}
    }{$n=10$}\hfill
    \stackunder[5pt]{
    \includegraphics[width=0.30\textwidth]{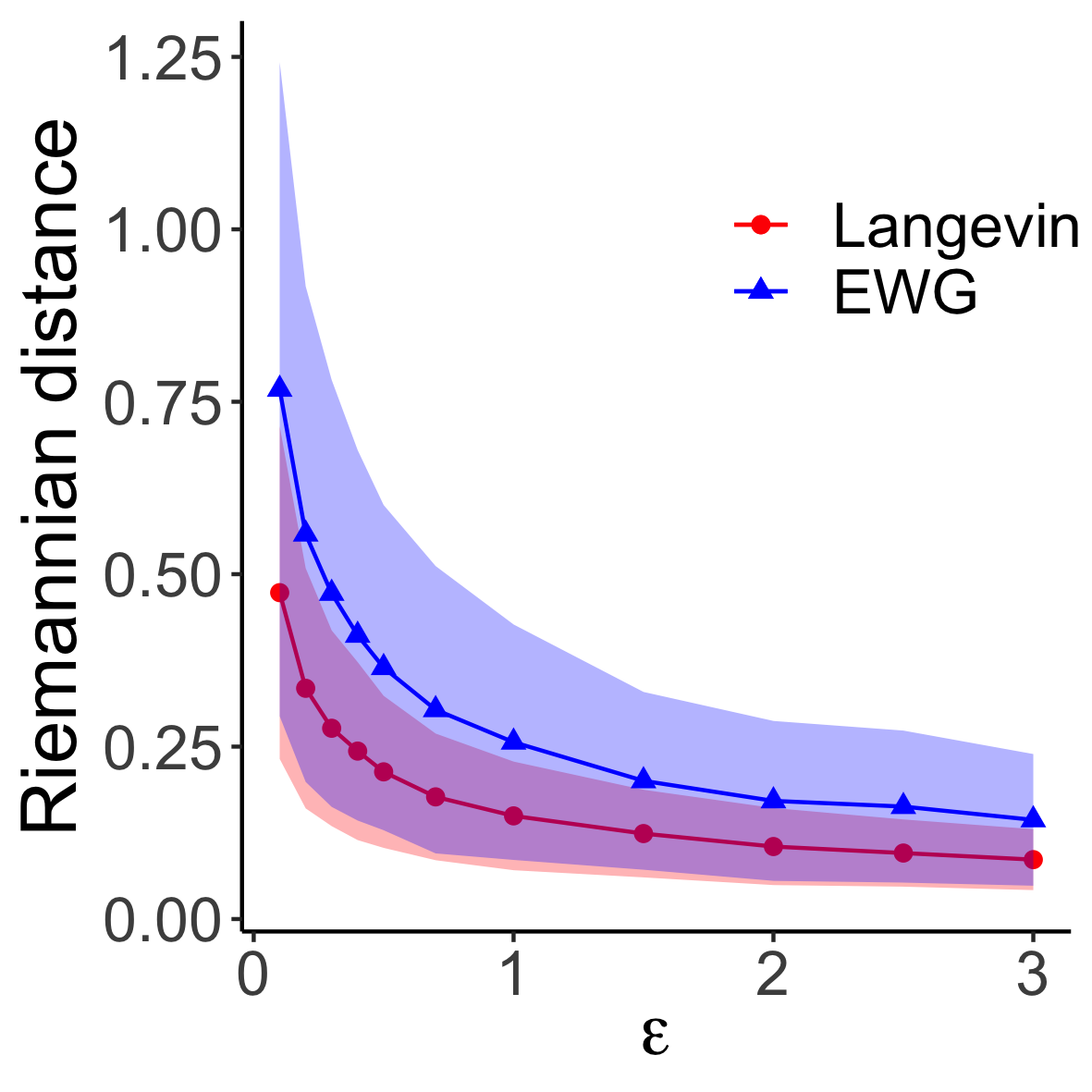}
    }{$n=50$}\hfill
    \stackunder[5pt]{
    \includegraphics[width=0.30\textwidth]{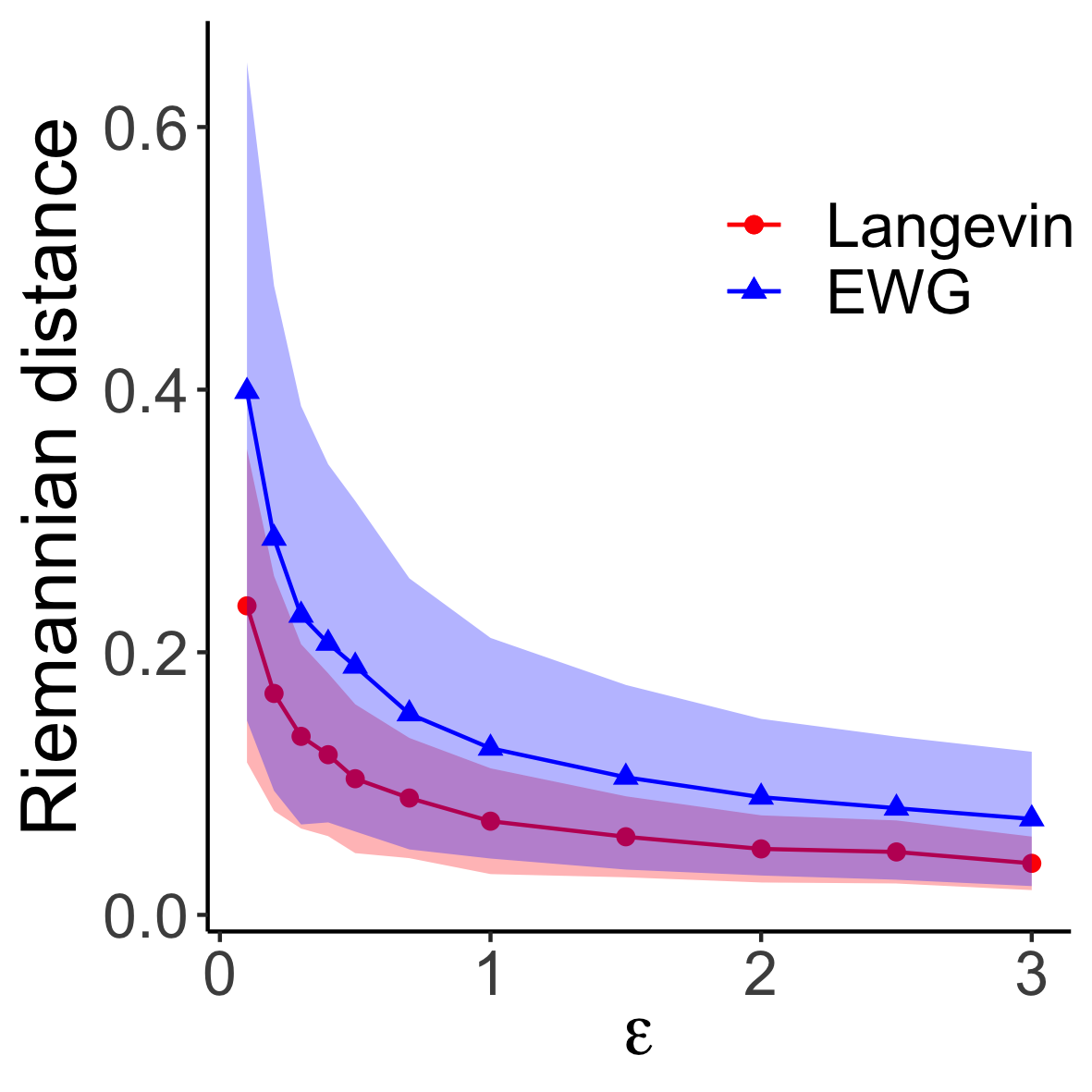}
    }{$n=100$}
    \caption{Utility comparison of Langevin and EWG mechanisms across different sample size $n \in \{10, 50, 100\}$. Blue lines with triangular symbols show the distances of the EWG mechanism, and red lines with circular symbols represent the Riemannian distances of our Langevin mechanism. The results are shown for Scenario 1 (top) and Scenario 2 (bottom).}
    \label{fig_hyperbolic_EWG}
\end{figure}

Figures~\ref{fig_hyperbolic_RL} and~\ref{fig_hyperbolic_EWG} reveal the same qualitative privacy--utility trade-off as in the spherical setting. Across all mechanisms, the Riemannian error decreases as the privacy budget $\varepsilon$ increases and as the sample size $n$ grows. Under Scenario~1, the three mechanisms perform comparably overall, with the RL mechanism showing a slight advantage over the Langevin mechanism, and the latter outperforming the EWG mechanism when both $n$ and $\varepsilon$ are small. Under Scenario~2, by contrast, the Langevin mechanism uniformly outperforms the EWG mechanism across all privacy budgets; see the bottom row of Figure~\ref{fig_hyperbolic_EWG}. This agrees with our theory: while both the Langevin and EWG mechanisms depend on the choice of an anchor or footpoint, only the Langevin mechanism benefits from the additional privacy protection induced by the selection step itself. The comparison between the RL and Langevin mechanisms is closer; as shown in the bottom row of Figure~\ref{fig_hyperbolic_RL}, the Langevin mechanism performs better for small $n$ and small $\varepsilon$, and remains competitive, though slightly less accurate, as $n$ and $\varepsilon$ become larger.

\section{Discussions and Future Works}\label{sec:discussion}

The diffusion-based mechanisms developed in this work open several directions for further research at the interface of differential privacy, stochastic analysis, and Riemannian statistics. While our results focus on one-shot releases and finite-sample sensitivity analysis, many natural questions arise when considering sequential mechanisms, long-time diffusion behavior, and asymptotic statistical properties of privatized estimators. In particular, it is important to understand how privacy accounting interacts with the semigroup structure of manifold-valued diffusions and how the injected diffusion noise affects the limiting distribution of geometric estimators. We highlight two directions that appear especially promising: gradual-release mechanisms based on diffusion trajectories, and asymptotic distribution theory for privatized generalized Fr\'echet means.

\subsection{Gradual Release, Stopping Rules, and Random Walks}
A natural next step is to move beyond the one-shot release studied here and develop a gradual-release version of our diffusion mechanisms. The semigroup and Markov structures of Brownian motion and Langevin dynamics suggest that one should be able to couple releases across multiple times, starting from a highly privatized statistic and then revealing progressively less noisy versions along the same diffusion trajectory. In such a framework, the main question is whether the total privacy cost can be controlled primarily by the final released time, or more generally by a stopping time chosen from the observed path, rather than by naively composing all intermediate releases. Establishing such a result would require a stopping-time or ex-post version of RDP accounting for manifold-valued diffusions, and it would be especially interesting to understand whether Harnack-type semigroup arguments can be made time-uniform in this sequential setting.

A closely related issue is the design of intrinsic stopping thresholds. In Euclidean gradual-release mechanisms, one may stop when a target utility level is reached \cite{whitehouse2022brownian}; on manifolds, the stopping rule should respect the geometry and may depend on quantities such as the empirical Fr\'echet functional, a geodesic prediction loss, successive inter-iterate distances, or curvature-adjusted confidence radii. When the threshold is evaluated on private data, one would further need an intrinsic analogue of private thresholding procedures. Since Algorithms~\ref{alg:bm_sampler} and \ref{alg:langevin_sampler} are implemented through intrinsic discretizations, it is also natural to ask whether the gradual-release idea can be analyzed directly at the level of geodesic random walks. Such a discrete-time theory would be practically important, because it could quantify how discretization error, path coupling, and privacy accounting interact in finite-step implementations.

\subsection{Central Limit Theorems for the Private Generalized Fr\'echet Mean}
A second direction is to study the asymptotic distribution of the private generalized Fr\'echet mean under our diffusion mechanisms. Our sensitivity bounds already suggest that the interaction between statistical fluctuation and privacy noise is highly nontrivial and may depend sharply on the exponent $p$. Under a fixed RDP budget, the diffusion time is calibrated through the global sensitivity of $\mu_p$, so the size of the injected noise is of the same order as the sensitivity bound. This suggests that when $\Delta_{\mu_p}=o(n^{-1/2})$, the privacy perturbation may be asymptotically negligible and the private estimator may inherit the same tangent space central limit theorem as the non-private generalized Fr\'echet mean. By contrast, in the regime $p>2$, our sensitivity bounds are of order $n^{-1/(p-1)}$, which points to a possible phase transition around $p=3$: for $2<p<3$ the classical $\sqrt{n}$-scale may remain dominant, for $p=3$ the privacy noise and sampling fluctuation may coexist on the same scale, and for $p>3$ the privacy term may dominate and lead to a different normalization.

Making this rigorous would require combining the geometry of generalized Fr\'echet means with a local analysis of the diffusion mechanisms in normal coordinates. In particular, it would be valuable to understand how curvature, cut locus effects, and possible smeariness phenomena influence the limiting law of the privatized estimator in tangent coordinates around the population Fr\'echet mean. One would also like to determine whether the limit is the same as in the non-private problem, a convolution of the classical limit with an intrinsic diffusion law, or a genuinely new non-Euclidean asymptotic regime. Results of this kind would clarify when diffusion-based privacy is asymptotically negligible, when it changes the effective rate of estimation, and how the geometry of the manifold interacts with privacy in large samples.

\begin{appendices}

\section{Supplemental material}
\par The supplemental materials contains detailed proofs of main theorems in this manuscript.

\end{appendices}
\bibliographystyle{abbrvnat}
\bibliography{aos}
\end{document}